%% file: SM_ECMLPKDD_MLJournal.tex
\RequirePackage{fix-cm}
\documentclass[smallcondensed]{svjour3_org}     % onecolumn (ditto)
\smartqed  % flush right qed marks, e.g. at end of proof
\usepackage{graphicx} % more modern
\usepackage{epstopdf}
\usepackage{enumitem}
\graphicspath{{./figs/}}
\usepackage{natbib}
\usepackage{wrapfig}
\usepackage{algorithm}
\usepackage{algorithmic}

\usepackage{tabularx}
\usepackage{amsmath}
\usepackage{amssymb}
\usepackage{comment}
\usepackage{caption}
\usepackage{booktabs} 
\usepackage{multirow}
\usepackage{ulem}

\usepackage{amsthm}

\newcommand{\ignore}[1]{}
\setcitestyle{aysep={},yysep={;}} 
\theoremstyle{plain}

\theoremstyle{definition}
\newtheorem{thm}{Theorem}[section]
\newtheorem{lem}[thm]{Lemma}

\theoremstyle{definition}

\theoremstyle{remark}

\newcommand\subparagraph{%
  \@startsection{subparagraph}{5}
  {\parindent}
  {3.25ex \@plus 1ex \@minus .2ex}
  {-1em}
  {\normalfont\normalsize\bfseries}} 
 \usepackage[small,compact]{titlesec}

\def \etc {{etc}}
\def \ie {{i.e., }}
\def\@thmcountersep{.}
\begin{document}

\title{Generalized Twin Gaussian Processes using \\ Sharma-Mittal Divergence%\thanks{Grants or other notes
}
\subtitle{}


\author{Mohamed Elhoseiny         \and
        Ahmed Elgammal %etc.
}


\institute{Mohamed Elhoseiny \at
              110 Frelinghuysen Road, \\ 
              Piscataway, NJ 08854-8019 \\
              USA\\
              Tel.: +1-732-208-9712\\	
              \email{m.elhoseiny@cs.rutgers.edu}           %  \\
           \and
           Ahmed Elgammal \at
           110 Frelinghuysen Road, \\ 
           Piscataway, NJ 08854-8019 \\
           USA\\
           \email{elgammal@cs.rutgers.edu}          
}

\date{Received: date / Accepted: date}

\maketitle

\begin{abstract} 
\input{abstract}


\end{abstract}

\section{Introduction}
\label{sec:intro}
\input{intro}

\section{Sharma-Mittal Divergence}
\label{sec:background}

\input{background}

\section{Sharma-Mittal TGP}
\label{sec:smTGP}

\input{smTGP}

\section{Theoretical Analysis}
\label{sec:eiganalsis}

\input{eiganalysis3}

\section{Experimental Results}
\label{sec:experiments}

\input{experiments3}

\section{Discussion and Conclusion}
\label{sec:conclusion}

\input{conclusion}

%\appendix
\section*{Appendices}
\input{appendices}

\bibliographystyle{spbasic}      % basic style, author-year citations
\bibliography{example_paper2}
%\bibliographystyle{icml2013}
%\begin{acknowledgements}
%\end{acknowledgements}



\ignore{

}
\end{document}

%% file: abstract.tex
There has been a growing interest in mutual information measures due to their wide range of applications in Machine Learning and Computer Vision.  In this paper, we present a generalized structured regression framework based on Shama-Mittal divergence, a relative entropy measure, which is introduced to the Machine Learning community in this work. Sharma-Mittal (SM) divergence is a generalized mutual information measure for the widely used	 R\'enyi, Tsallis, Bhattacharyya, and Kullback-Leibler (KL) relative entropies. Specifically, we study Sharma-Mittal divergence as a cost function in the context of the Twin Gaussian Processes (TGP)~\citep{Bo:2010}, which generalizes over the KL-divergence without computational penalty.  We show interesting properties of Sharma-Mittal TGP (SMTGP) through a theoretical analysis, which covers missing insights in the traditional TGP formulation. However, we generalize this theory based on SM-divergence instead of KL-divergence which is a special case. Experimentally, we evaluated the proposed SMTGP framework on several datasets. The results show that SMTGP  reaches better predictions than KL-based TGP, since it offers a bigger class of models through its parameters that we learn from the data.

%% file: intro.tex
\begin{sloppypar}
Since 1950s, a lot of work has been done to measure information and probabilistic metrics. Claude Shannon~\citep{Shannon:2001} proposed a powerful framework to mathematically quantify \ignore{the intuition behind }information \ignore{quantity}, which has been the foundation of the information theory and the development in communication, networking, and a lot of Computer Science applications. Many problems in Physics and Computer Science require a reliable measure of information divergence, which have motivated many mathematicians, physicists, and computer scientists to study different divergence measures. For instance, R\'enyi~\citep{Renyi60}, Tsallis~\citep{Tsallis88} and Kullback-Leibler divergences~\citep{KLRef90} have been applied in many \ignore{physics (such as~\citep{IsSM07}) and }Computer Science applications. They have been effectively used in machine learning for many tasks including subspace analysis~\citep{Learned-Miller:2003,Poczos:2005,VanHulle:2008,Szabó07a}, facial expression recognition~\citep{Shan2005}, texture classification~\citep{Hero01alpha-divergencefor}, image registration~\citep{kybic06}, clustering~\citep{agha07}, non-negative matrix factorization~\citep{Yu2013} and  3D pose estimation~\citep{Bo:2010}.  
\end{sloppypar}

In the Machine Learning community, a lot of attempts have been done  to understand information and connect it to uncertainty. Many of proposed terminologies turns out to be different views of the same measure. For instance, Bregman Information~\citep{Banerjee:2005}, Statistical Information~\citep{DeGroot62}, Csiszár-Morimoto f-divergence, and the gap between the expectations in Jensen's inequality (\ie    the Jensen gap)~\citep{jensen06} turn out to be equivalent to the maximum reduction in uncertainty for convex functions, in contrast with the prior probability distribution~\citep{Reid:2011}\ignore{; ``Mathematics is the Art of Giving the Same Name to Different Things, Henri Poincare''}. 

\begin{sloppypar}
A lot of work has been proposed in order to unify divergence functions~\citep{AmNa00,Reid:2011,Zhang2007,Zhang2004}. \ignore{However, Cichoki et al}~\citet{CichockiA10} considered explicitly  the relationships between Alpha-divergence~\citep{CichockiLKCAlpha08}, Beta-divergence~\citep{KompassBeta07} and Gamma-divergence~\citep{CichockiA10}; each of them is a single-parameter divergence measure. Then, \ignore{Cichocki et al }~\citet{Cichocki11} introduced a two-parameter family\ignore{~\citep{Cichocki11}}. However, we study here a two-parameter divergence measure~\citep{SM75}, investigated in  the Physics community, which is interesting to be considered in the Machine Learning community.  
\end{sloppypar}

 \ignore{Akt\"urk et al }~\citet{IsSM07}, physicists\footnote{ This work was proposed four years before~\citet{Cichocki11} and it was not considered either as a prior work in the Machine Learning community as far as we know},  studied an entropy measure called Sharma-Mittal on theormostatics in 2007, which was originally introduced by Sharma BD et al~\citep{SM75}. Sharma-Mittal (SM) divergence has two parameters ($\alpha$ and $\beta$), detailed later in Section~\ref{sec:background}.\ignore{Akt\"urk et al~\citep{IsSM07}}~\citet{IsSM07} discussed that SM entropy generalizes both Tsallis ($\beta \to \alpha$) and R\'enyi entropy ($\beta \to 1$) in the limiting cases of its two parameters; this was originally showed by~\citep{sb05}. In addition, it can be shown that SM entropy converges to Shannon entropy as $\alpha, \beta \to 1$.  Akt\"urk et al also suggested a physical meaning of SM entropy, which is the free energy difference between the equilibrium and the off-equilibrium distribution. \ignore{Through thermostatic experiments, they showed that SM entropy acquires this interpretation in some limit cases. }In 2008, SM entropy was also investigated in multidimensional harmonic oscillator systems~\citep{thph08}. Similarly, SM  relative entropy (mutual information) generalizes each of the R\'enyi, Tsallis and KL mutual information divergences. This work in physics domain motivated us to investigate SM Divergence in the Machine Learning domain.

%In this paper, we study Sharma-Mittal divergence/relative entropy for structured prediction. 
\begin{sloppypar}
A closed-form expression for SM divergence  between two Gaussian distributions was recently proposed~\citep{SM:2012}, which motivated us to study this measure in structured regression setting. In this paper, we present a generalized framework for structured regression utilizing a family of divergence measures that includes SM divergence, R\'enyi divergence, Tsallis divergence  and KL divergence. In particular, we study SM divergence within the context of Twin Gaussian Processes (TGP), a state-of-the-art structured-output regression method.\ignore{ Liefeng Bo~\citep{Bo:2010}}~\citet{Bo:2010} proposed TGP as a structured prediction approach based on estimating the KL divergence from the input  to output Gaussian Processes, denoted by KLTGP\footnote{ that is why it is called Twin Gaussian Processes}. Since KL divergence is not symmetric,\ignore{ Liefeng Bo~\citep{Bo:2010}}~\citet{Bo:2010} also studied TGP based on KL divergence from the output to the input data, denoted by IKLTGP (Inverse KLTGP). In this work, we present a generalization for TGP using the SM divergence, denoted by SMTGP. Since SM divergence is a two-parameter family, we study the effect of these parameters and how they are related to the distribution of the data. In the context TGP, we show that these two parameters, $\alpha$ and $\beta$, could be interpreted as distribution bias and divergence order in the context of structured learning. We also highlight probabilistic causality direction of the SM objective function\footnote{This is mainly detailed in section~\ref{sec:eiganalsis}}.  More specifically, there are six contributions to this paper
\end{sloppypar}
\begin{enumerate}[noitemsep,topsep=0pt,parsep=0pt,partopsep=0pt,leftmargin=*]
\item The first presentation of SM divergence in the Machine Learning Community
\item A generalized version of TGP based on of SM divergence to predict structured outputs; see Subsections~\ref{sec:SMTGP}.
\item A simplification to the SM divergence closed-form expression in~\citep{SM:2012} for Multi-variate Gaussian Distribution\footnote{This simplification could be useful out of the context  TGP, while computing SM-divergence between two multi-variate distributions}, which reduced both the cost function evaluation and  the gradient computation, used in our prediction framework; see Subsections~\ref{sec:ISMTGP} and~\ref{ss:ncfa}.
\item Theoretical analysis of TGP under SM divergence in  Section~\ref{sec:eiganalsis}.
\item A certainty measure, that could be associated with each structured output prediction, is argued in subsection~\ref{sec:tgpcostanalysis}.
\item An experimental demonstration that SM divergence improves on KL divergence under TGP prediction by correctly tuning $\alpha$ and $\beta$ through cross validation on two toy examples and three real datasets; see Section~\ref{sec:experiments}.
\end{enumerate}

\vspace{3mm}

The rest of this paper is organized as follows: Section~\ref{sec:background} presents background on SM Divergence and its available closed-form expression for multivariate Gaussians. Section~\ref{sec:smTGP} presents the optimization problem used in our framework and the derived analytic gradients. \ignore{In addition, it shows an argument on the performance of SMTGP  by comparing the existing closed-form expressions of SM divergence and a simplified version that we derived;} Section~\ref{sec:eiganalsis} presents our theoretical analysis on TGP under our framework from spectral perspective. Section~\ref{sec:experiments} presents our experimental validation. Finally, Section~\ref{sec:conclusion} discusses and concludes our work.

%% file: background.tex
This section addresses a background on SM-divergence and its closed form for the multivariate Gaussian distribution.

\subsection{SM Family Divergence Measures}
\label{ss:smd}
The SM divergence, $D_{\alpha, \beta}(p:q)$,  between two distributions $p(t)$ and $q(t)$ is defined as~\citep{SM75}
\begin{equation}
\small
\begin{split}
D_{\alpha, \beta}(p:q) = &\frac{1}{\beta-1} (\int_{-\infty}^{\infty}{p(t)^\alpha q(t)^{1-\alpha} dt} )^\frac{1-\beta}{1-\alpha} -1), \forall \alpha >0 , \alpha \neq 1 , \beta \neq 1. \\
\end{split}
\label{eq:smgen}
\end{equation}
It was shown in~\citep{IsSM07} that most of the widely used divergence measures are special cases of SM divergence. Each of the R\'enyi, Tsallis and KL divergences can be defined as limiting cases of SM divergence as follows:
\begin{equation}
\small
\begin{split}
%D_{\alpha, \beta}(p:q) = &\frac{1}{\beta-1} (\int_{-\infty}^{\infty}{p(t)^\alpha q(t)^{1-\alpha} dt} )^\frac{1-\beta}{1-\alpha} -1), \\
%& \forall \alpha >0 , \alpha \neq 1 , \beta \neq 1 \\
R_{\alpha}(p:q) = &\lim_{\beta \to 1} {D_{\alpha, \beta}(p:q) }  = \frac{1}{\alpha-1} ln (\int_{-\infty}^{\infty}{p(t)^\alpha q(t)^{1-\alpha} dt} )),  	\forall \alpha >0 , \alpha \neq 1.  \\
T_{\alpha}(p:q) = &{D_{\alpha, \alpha}(p:q) }= \frac{1}{\alpha-1} (\int_{-\infty}^{\infty}{p(t)^\alpha q(t)^{1-\alpha} dt} ) -1),  \forall \alpha >0 , \alpha \neq 1 , \\
KL(p:q) = & \lim_{\beta \to 1, \alpha \to 1} {D_{\alpha, \beta}(p:q) } = \int_{-\infty}^{\infty}{p(t). ln ( \frac{p(t)}{q(t)} dt}) 
\end{split}
\label{eqdef}
\end{equation}
where  $R_{\alpha}(p:q)$, $T_{\alpha}(p:q)$ and $KL(p:q)$ denotes R\'enyi, Tsallis, KL divergences respectively. We also found that Bhattacharyya divergence~\citep{bhatt67}, denoted by $B(p:q)$   is a limit case of SM and R\'enyi divergences as follows 
\begin{equation*}
\small
\begin{split}
B(p:q)& = 2 \cdot \lim_{\beta \to 1, \alpha \to 0.5}  D_{\alpha,\beta}(p:q) = 2 \cdot \lim_{\alpha \to 0.5}  R_{\alpha}(p:q) = - ln \Big( \int_{-\infty}^{\infty} p(x)^{0.5} q(x)^{0.5} dx \Big).
\end{split}
\end{equation*}

\begin{sloppypar}
While SM is a two-parameter  generalized entropy measure originally introduced by\ignore{ Sharma \& Mittal~\citep{SM75} }~\citet{SM75}\ignore{ in 1975}, it is worth to mention that two-parameter family of divergence functions has been recently proposed in the machine learning community since 2011~\citep{Cichocki11,Zhang2013}. It is shown  in~\citep {CichockiA10} that the Tsallis entropy is connected  to the Alpha-divergence~\citep{CichockiLKCAlpha08}, and Beta-divergence~\citep{KompassBeta07}\footnote{ Alpha and Beta divergence should not be confused with $\alpha$ and $\beta$ parameters of Sharma Mittal divergence}, while
the R\'enyi entropy is related to the Gamma-divergences~\citep{CichockiA10}. The  Tsallis and R\'enyi  relative entropies are two different generalization of the standard Boltzmann-Gibbs entropy (or Shannon information). However, we focus here on SM divergence for three reasons (1) It generalizes over a considerable family of functions suitable for structured regression problems (2) Possible future consideration of this measure in works that study entropy and divergence functions, (3) SM divergence has a closed-form expression, recently proposed for multivariate Gaussian distributions~\citep{SM:2012}, which is interesting to study. 
\end{sloppypar}

\ignore{
\begin{equation}
\small
\begin{split}
D_{\alpha', \beta'}(p:q) = &\frac{1}{\alpha \beta} (\int_{-\infty}^{\infty}{p(t)^\alpha q(t)^{\beta} - \frac{\alpha}{\alpha+\beta} p(t)^{\alpha + \beta} - \frac{\beta}{\alpha+\beta} q(t)^{\alpha+\beta}  dt} )
\end{split}
\label{eq:smgen}
\end{equation}
}

\ignore{{\noindent{\bf Motivation:}} One of the main} Another motivations of this work is to study how the two parameters of the SM  Divergence, as a  generalized entropy measure, affect the performance of the structured regression problem. Here we show an analogy in the physics domain that motivates our study. As indicated  by~\citet{sb05} in physics domain, it is important to understand that Tsallis and R\'enyi entropies are two different generalizations along two different paths. Tsallis generalizes to non-extensive systems\footnote{\ie In Physics, Entropy is considered to have an extensive property if its value depends on the amount of material present; Tsallis is an non-extensive entropy}, while R\'enyi to quasi-linear means\footnote{\ie R\'enyi  entropy is could be interpreted as an averaging of quasi-arithmetic function ~\citet{IsSM07}}. SM entropy generalizes to non-extensive sets and non-linear means having Tsallis and R\'enyi measures as limiting cases. Hence, in TGP regression setting, this indicates resolving the trade-off  of  having a control of the direction of bias towards one of the distributions (i.e. input and output distributions) by changing $\alpha$. It also allows higher-order divergence measure by changing $\beta$. Another motivation from Physics is that SM entropy is the only entropy that gives rise to a thermostatistics based on escort mean values\footnote{ escort mean values are useful theoretical tools, used in thermostatistics,for describing basic properties of some probability density function ~\citep{phcit09}} and admitting of a partition function~\citep{sm02}. \ignore{In addition, our motivation is based on the fact that SM entropy is an entropy measure that gives rise to a thermostatistics based on escort mean values and admitting of a partition function~\citep{sm02}.}

\subsection{SM-divergence Closed-Form Expression for Multivariate Gaussians}
\begin{sloppypar}
In order to solve optimization problems efficiently over relative entropy, it is critical to have a closed-form formula for the optimized function, which is SM relative entropy in our framework. Prediction over Gaussian Processes~\citep{Rasmussen:2005} is performed practically as a multivariate Gaussian distribution. Hence,  we are interested in finding a closed-form formula for SM relative entropy of distribution $\mathcal{N}_q$ from $\mathcal{N}_p$, such that $\mathcal{N}_p = \mathcal{N}(\mu_p, \Sigma_p)$, and $\mathcal{N}_q = \mathcal{N}(\mu_q, \Sigma_q)$. In 2012, Frank Nielsen proposed a closed form expression for SM divergence~\citep{SM:2012} as follows
\end{sloppypar}
\begin{equation}
\small
\begin{split}
D_{\alpha, \beta}(\mathcal{N}_p:\mathcal{N}_q) =  
\frac{1}{\beta-1} \Big[&\Big(\frac{|\Sigma_p|^{\alpha} {|\Sigma_q}|^{1 -\alpha}}{|(\alpha {\Sigma_p}^{-1} +(1-\alpha) {\Sigma_q}^{-1})^{-1}|} \Big)^{-\frac{1-\beta}{2(1-\alpha)}} \cdot \\
& e^{-\frac{\alpha (1-\beta)}{2}   \Delta\mu^T (\alpha {\Sigma_p}^{-1} +(1-\alpha) {\Sigma_q}^{-1})^{-1} \Delta\mu }  -1\Big] 
\end{split}
\label{eqcfsm1}
\end{equation}
where  $0\le \alpha \le1$, $\Delta\mu = \mu_p - \mu_q$,  $\alpha {\Sigma_p}^{-1} +(1-\alpha) {\Sigma_q}^{-1}$ is a positive definite matrix, and $| \cdot |$ denotes the matrix determinant. The following section builds on this SM closed-form expression to predict structured output under TGP\ignore{ with some enhancements to increase the speed of optimization}, which leads an analytic gradient of the SMTGP cost function with cubic computational complexity. We then present a simplified expression of the closed-form expression in Equation~\ref{eqcfsm1}, which results in an equivalent SMTGP analytic gradient of quadratic complexity.

%% file: smTGP.tex
In prediction problems, we expect that similar inputs \ignore{to }produce similar predictions. This notion was adopted in~\citep{Bo:2010,Yamada:2012} to predict structured output based on KL divergence between two Gaussian Processes. This section presents TGP for structured regression by minimizing SM relative entropy. We follow that by our theoretical analysis of TGPs in Section~\ref{sec:eiganalsis}. We begin by introducing some notation. Let the joint distributions of the input and the output be defined as follows
\begin{equation}
\small
\begin{split}
&p(X,x) = \mathcal{N}_X(0, K_{X\cup x}) ,  p(Y,y) = \mathcal{N}_Y(0, K_{Y \cup y}), \\
&K_{X\cup x} = \begin{bmatrix}
K_X & K_X^x\\ 
 {K_X^x}^T& \it{K_X}(x,x)
\end{bmatrix}, K_{Y\cup y} = \begin{bmatrix}
K_Y & K_Y^y\\ 
 {K_Y^y}^T& \it{K_Y}(y,y)
\end{bmatrix} \\
\end{split}
\end{equation}
where $x_{(d_x \times 1)}$  is a new input test point,  whose unknown outcome is $y_{(d_y \times 1)}$ and the training set is $X_{(N\times d_x)}$ and $Y_{(N \times d_y)}$ matrices.  $K_X$ is an $N \times N$ matrix with $(K_X)_{ij} = k_X(x_i, x_j )$, such that $k_X(x_i, x_j )$ is the similarity kernel between $x_i$ and $x_j$.  $K_X^x$ is an $N \times 1$ column vector with $(K_X^x)_i = k_X(x_i, x)$. Similarly, $K_Y$ is an $N \times N$ matrix with $(K_Y)_{ij} = k_Y(y_i, y_j )$, such that $k_Y(y_i, y_j )$ is the similarity kernel between $y_i$  and $y_j$, and $K_Y^y$ is an $N \times 1$ column vector with $(K_Y^y)_i = k_Y(y_i, y)$. By applying  Gaussian-RBF kernel functions, the similarity kernels for inputs and outputs will be in the form of  $k_X(x_i,x_j) = exp(\frac{- \|x_i-x_j\|^2}{2 \rho_x^2})+ \lambda_X  \delta_{ij}$ and $k_Y(y_i,y_j)= exp(\frac{- \|y_i-y_j\|^2}{2 \rho_y^2})+ \lambda_Y  \delta_{ij}$, respectively, where $\rho_x$ and $\rho_y$ are the corresponding kernel bandwidths, $\lambda_X$ and $\lambda_Y$ are regularization parameters to avoid overfitting and to handle noise in the data, and $\delta_{ij} =1$ if $i=j$, $0$ otherwise. 

\subsection{KLTGP and IKLTGP Prediction}
\ignore{Bo et al~\citep{Bo:2010}}~\citet{Bo:2010} firstly proposed TGP which minimizes the Kullback-Leibler divergence between the marginal GP of inputs and outputs. However, they were focusing on the Human Pose Estimation problem. As a result, the estimated pose using TGP is given as the solution of the following optimization problem~\citep{Bo:2010}
\begin{equation}
\small
\begin{split}
\hat{y} =  \underset{y}{\operatorname{argmin       }}[ & L_{KL}(x,y) =  k_Y(y,y)  -2 {K_Y^y}^T u_x - \eta_x log (k_Y(y,y) - {K_Y^y}^T (K_Y)^{-1} K_Y^y ) ]
\end{split}
\label{eq:tgp}
\end{equation}
where $u_x = (K_X)^{-1} {K_X^x}$, $\eta_x  = k_X(x,x) -{K_X^x}^T  u_X $. The analytical gradient of this cost function is defined as follows~\citep{Bo:2010}
\begin{equation}
\small
\begin{split}
\frac{\partial  L_{KL}(x,y)}{\partial  y^{(d)}} &= \frac{\partial   k_Y(y,y)}{\partial  y^{(d)}} - 2  u_x^T \frac{\partial K_Y^y}{\partial  y^{(d)}} - \eta_x \frac{log(\frac{\partial   k_Y(y,y)}{\partial  y^{(d)}} - 2 {K_Y^y}^T (K_Y)^{-1} \frac{K_Y^y}{\partial  y^{(d)}})}{k_Y(y,y) - {K_Y^y}^T (K_Y)^{-1} K_Y^y }
\end{split}
\label{eq:tgpgrad}
\end{equation}
where  $d$ is the dimension index of the output $y$. For Gaussian kernels, we have \\   \begin{equation*}
\frac{\partial  k_Y(y,y)}{\partial  y^{(d)}}  = 0 ,  \frac{\partial {K_Y^y}}{ \partial {y^{(d)}}} = \begin{bmatrix}
-\frac{1}{\rho_y^2} (y^{(d)}- y_1^{(d)}) k_Y(y, y_1)  \\
-\frac{1}{\rho_y^2} (y^{(d)}- y_2^{(d)}) k_Y(y, y_2)  \\
... \\
-\frac{1}{\rho_y^2} (y^{(d)}- y_N^{(d)}) k_Y(y, y_N)  \\
\end{bmatrix}.
\end{equation*}

The optimization problem can be solved using a second order BFGS quasi-Newton optimizer with cubic polynomial line search for optimal step size selection. Since KL divergence is not symmetric, \ignore{Bo et al }~\citet{Bo:2010} also studied inverse KL-divergence between the output and the input distribution under TGP; we denote this model as IKLTGP. Equations~\ref{eq:itgp}  and~\ref{eq:itgpgrad} show the IKLTGP cost function and its corresponding gradient\footnote{ we derived this equation since it was not provided in~\citep{Bo:2010}}.

\begin{equation}
\small
\begin{split}
\hat{y} =&  \underset{y}{\operatorname{argmin       }}[  L_{IKL}(x,y) =  -2 {K_X^x}^T u_y+  u_y^T K_X  u_y +  \eta_y ( log(\eta_y) - log(\eta_x))], \\
& u_y = K_Y^{-1} {K_Y^y}, \eta_y  = k_Y(y,y) -{K_Y^y}^T  u_y 
\end{split}
\label{eq:itgp}
\end{equation}

\begin{equation}
\small
\begin{split}
\frac{\partial  L_{IKL} (x,y)}{\partial  y^{(d)}} &= -2 {K_X^x}^T  K_Y^{-1} \frac{\partial {K_Y^y}}{ \partial {y^{(d)}}} + 2 u_y^T K_X  K_Y^{-1} \frac{\partial {K_Y^y}}{ \partial {y^{(d)}}} \\ & - 2   ( log(\eta_y) - log(\eta_x)+1)  {K_Y^y}^T    K_Y^{-1} \frac{\partial {K_Y^y}}{ \partial {y^{(d)}}}  
\end{split}
\label{eq:itgpgrad}
\end{equation}

From Equations~\ref{eq:tgpgrad} and~\ref{eq:itgpgrad}, it is not hard to see that the gradients of  KLTGP and IKLTGP can be computed in {quadratic} complexity, given that $K_X^{-1}$ and $K_Y^{-1}$ \ignore{(i.e. $O(N^3)$)} are precomputed once during training and stored, as it depends only on the training data. This quadratic complexity of KLTGP gradient presents a benchmark for us to compute the gradient for SMTGP in $O(N^2)$. Hence, we address this benchmark in our framework, as detailed in the following subsections.

\subsection{SMTGP Prediction}
\label{sec:SMTGP}
By applying the closed-form in Equation~\ref{eqcfsm1},  SM divergence between $p(X,x)$ and $p(Y,y)$ becomes in the following form
\begin{equation}
\small
\begin{split}
& D_{\alpha, \beta}(p(X,x) : p(Y,y))= \frac{1}{\beta-1} \Big[\Big(\frac{|K_{X\cup x}|^{\alpha} {|K_{Y \cup y}|}^{1 -\alpha}}{|(\alpha K_{X\cup x}^{-1} +(1-\alpha) {K_{Y \cup y}}^{-1})^{-1}|}\Big)^{-\frac{1-\beta}{2(1-\alpha)}}  -1\Big] \\
\end{split}
\label{smtgp11}
\end{equation}
From matrix algebra, $|K_{X\cup x}| = |K_X| (k_X(x,x) - {K_X^x}^T {K_X}^{-1} {K_X^x})$. Similarly, $|K_{Y \cup y}|$ $= |K_Y| (k_Y(y,y) - {K_Y^y}^T {K_Y}^{-1} {K_Y^y})$. Hence, Equation~\ref{smtgp11} could be rewritten as follows
\begin{equation}
\small
\begin{split}
  D_{\alpha, \beta}(p(X,x) : p(Y,y))= &\frac{|K_X|^\frac{-\alpha(1-\beta)}{2(1-\alpha)}  |K_Y|^\frac{-(1-\beta)}{2}}{\beta -1}  \cdot 
(k_X(x,x) - {K_X^x}^T {K_X}^{-1} {K_X^x})^\frac{-\alpha(1-\beta)}{2(1-\alpha)}  \cdot\\
 (k_Y(y,y) - {K_Y^y}^T& {K_Y}^{-1} {K_Y^y})^\frac{-(1-\beta)}{2}     \cdot   |\alpha K_{X\cup x}^{-1} +(1-\alpha) {K_{Y \cup y}}^{-1}|^\frac{-(1-\beta)}{2(1-\alpha)} - \frac{1}{\beta-1}
\end{split}
\end{equation}
$|K_X|^\frac{-\alpha(1-\beta)}{2(1-\alpha)}  |K_Y|^\frac{-(1-\beta)}{2} $ is a positive constant, since $K_X$ and $K_Y$ are positive definite matrices. Hence, it could be removed from the optimization problem. Same argument holds for $|K_{X \cup x}|=|K_X| (k_X(x,x) - {K_X^x}^T {K_X}^{-1} {K_X^x})>0$, so $(k_X(x,x) - {K_X^x}^T {K_X}^{-1} {K_X^x})>0$ could be also removed from the cost function. Having removed these constants, the prediction function reduces to minimizing the following expression
\begin{equation}
\small
\begin{split}
L_{\alpha, \beta}(p(X,x) : p(Y,y)) = & \frac{1}{\beta-1} (k_Y(y,y) - {K_Y^y}^T {K_Y}^{-1} {K_Y^y})^\frac{-(1-\beta)}{2}     \cdot \\ &    |\alpha K_{X\cup x}^{-1} +(1-\alpha) {K_{Y \cup y}}^{-1}|^\frac{-(1-\beta)}{2(1-\alpha)} \\
\end{split}
\label{eq:lsmd1}
\end{equation}
It is worth mentioning that $K_{X\cup x}^{-1}$ is quadratic to compute, given that $K_X^{-1}$ is precomputed during the training; see Appendix A.

To avoid numerical instability problems in Equation~\ref{eq:lsmd1} (introduced by determinant of the large matrix $(\alpha K_{X\cup x}^{-1} +(1-\alpha) {K_{Y \cup y}}^{-1})$, we optimized $\log (L_{\alpha, \beta}(N_X : N_Y))$ instead of $L_{\alpha, \beta}(N_X : N_Y)$. We derived the gradient of $\log (L_{\alpha, \beta}(N_X : N_Y))$ by applying the matrix calculus directly on the logarithm of Equation~\ref{eq:lsmd1}, presented below; the  derivation steps are detailed in Appendix B
\begin{equation}
\small
\begin{split}
%&\frac{\partial \log L_{\alpha, \beta}(p(X,x) : p(Y,y))}{\partial y^{(d)}}  = \frac{-(1-\beta)}{2}  \frac{(- 2 \cdot {K_Y^y}^T {K_Y}^{-1} \frac{\partial K_Y^y}{\partial y^{(d)}})}{(k_Y(y,y) - {K_Y^y}^T {K_Y}^{-1} {K_Y^y})}  -  \frac{-(1-\beta)}{2} \cdot 2  \cdot \mu_y^T  \cdot  \frac{\partial {K_Y^y}}{\partial y^{(d)}}  \\
&\frac{\partial L_{\alpha, \beta}(p(X,x) : p(Y,y))}{\partial y^{(d)}}  = {(1-\beta)} \bigg[ \frac{ {K_Y^y}^T {K_Y}^{-1} \frac{\partial K_Y^y}{\partial y^{(d)}}}{(k_Y(y,y) - {K_Y^y}^T {K_Y}^{-1} {K_Y^y})}  +   \mu_y^T  \cdot  \frac{\partial {K_Y^y}}{\partial y^{(d)}} \bigg]  \\
\end{split}
\label{eq:gradL}
\end{equation}
\ignore{where $\Big(\alpha K_{Y \cup y} K_{X\cup x}^{-1} K_{Y \cup y} +(1-\alpha) {K_{Y \cup y}} \Big)\mu_y = [0,0,...0,1]^T$.}
$\mu_y$ is computed by solving the following linear system of equations \begin{math}
\small
 \Big(\alpha K_{Y \cup y} K_{X\cup x}^{-1}\end{math} \begin{math}
\small K_{Y \cup y} +(1-\alpha) {K_{Y \cup y}} \Big)\mu^{'}_y = [0,0,...0,1]^T
\end{math},  $\mu_y$ is the first $N$ elements in $\mu^{'}_y$, which is a vector of $N+1$ elements.
 The computational complexity of the gradient in Equation~\ref{eq:gradL}\ignore{ $\frac{\partial log L(\alpha,\beta)}{\partial y(d)}$} is cubic at test time, due to solving this system. On the other hand, the gradient for KLTGP is quadratic. This problem motivated us to investigate the cost function to achieve a quadratic complexity of the gradient computation for SMTGP.

\subsection{Quadratic SMTGP Prediction}
\label{sec:ISMTGP}
We start by simplifying the closed-form expression introduced in~\citep{SM:2012}, which led to the $O(N^3)$ gradient computation.

\begin{lem}
SM-divergence between two N-dimensional multivariate Gaussians $\mathcal{N}_p = \mathcal{N}(0, \Sigma_p)$ and $\mathcal{N}_q = \mathcal{N}(0, \Sigma_q)$ can be written as
\label{lemma11}
\end{lem}
\begin{equation}
\small
\begin{split}
D'_{\alpha, \beta}&(\mathcal{N}_p:\mathcal{N}_q) = \frac{1}{\beta-1} \Bigg[ \Big(\frac{|\Sigma_p| ^{1-\alpha}|{\Sigma_q}|^\alpha}{  | \alpha \Sigma_q +(1-\alpha) { \Sigma_p} |}\Big)^{\frac{(1-\beta)}{2 (1-\alpha)}} -1   \Bigg]
\end{split}
\label{eq:Normsm2}
\end{equation}

\begin{proof}
Under TGP setting,  the exponential term in Equation~\ref{eqcfsm1} vanishes to 1, since $\Delta \mu =0$ (i.e. $\mu_p = \mu_q=0$). Then, $\frac{|\Sigma_p|^{\alpha} {|\Sigma_q}|^{1 -\alpha}}{|(\alpha \Sigma_p^{-1} +(1-\alpha) {\Sigma_q}^{-1})^{-1}|}$ could be simplified as follows:
\begin{equation}
\small
\begin{split}
&= \frac{{|\Sigma_p|}^{\alpha} {|\Sigma_q|}^{1 -\alpha}}{|(\alpha \Sigma_p^{-1} +(1-\alpha) {\Sigma_q}^{-1})|^{-1}}\text{, since }|A^{-1}|= \frac{1}{|A|}\\
&= \frac{{|\Sigma_p|}^{\alpha} {|\Sigma_q|}^{1 -\alpha}}{|\Sigma_p^{-1} (\alpha \Sigma_q +(1-\alpha) { \Sigma_p} ){\Sigma_q^{-1}|^{-1}}} \text{, by factorization}\\
&= \frac{{|\Sigma_p|}^{\alpha} {|\Sigma_q|}^{1 -\alpha}}{|\Sigma_p  || \alpha \Sigma_q +(1-\alpha) { \Sigma_p}|^{-1}|{\Sigma_q}|}\text{, since }|A B| =  |A| |B| \\
&= \frac{|\alpha \Sigma_q +(1-\alpha) {\Sigma_p} |}{{|\Sigma_p| }^{1-\alpha}{|\Sigma_q|}^\alpha} \text{, by rearrangement}
\end{split}
\label{eq:ncfe} 
\end{equation}
\end{proof}

We denote the original closed-form expression as $D_{\alpha,\beta}(\mathcal{N}_p, \mathcal{N}_q)$, while the simplified form $D'_{\alpha,\beta}(\mathcal{N}_p, \mathcal{N}_q)$\ignore{, where $\Delta \mu =0$}. After applying the simplified SM expression in Lemma~\ref{lemma11} to measure the divergence between $p(X,x)$ and $p(Y,y)$, the new cost function becomes in the following form
\begin{equation}
\small
\begin{split}
D'_{\alpha, \beta}(p(X,x) : p(Y,y))& =  \frac{1}{\beta-1} \Big[\big(\frac{|K_{X\cup x}|^{1-\alpha} {|K_{Y \cup y}|}^{\alpha}}{|(1-\alpha) K_{X\cup x} + \alpha {K_{Y \cup y}}|})^{\frac{1-\beta}{2(1-\alpha)}} -1 \Big]
\\  =  \frac{1}{\beta-1} \big(&{|K_{X\cup x}|^\frac{1-\beta}{2}  {|K_{Y \cup y}|}^\frac{\alpha (1-\beta)}{2(1- \alpha)}} {|(1-\alpha) K_{X\cup x} + \alpha {K_{Y \cup y}}|^\frac{-(1-\beta)}{2 (1-\alpha)}}) - \frac{1}{\beta-1}, \\
{|K_{Y \cup y}|}^\frac{\alpha (1-\beta)}{1-\alpha}  = & {|K_{Y}|}^\frac{\alpha (1-\beta)}{(1-\alpha)}  \cdot   (k_Y(y,y) - {K_Y^y}^T {K_Y}^{-1} {K_Y^y})^\frac{\alpha (1-\beta)}{(1-\alpha)} , \\
|(1-\alpha) K_{X\cup x} +& \alpha {K_{Y \cup y}}|^\frac{-(1-\beta)}{2 (1-\alpha)} = |(1-\alpha) K_{X} + \alpha {K_{Y}}|^\frac{-(1-\beta)}{2 (1-\alpha)} \cdot \\ & {({K_{xy}}^{\alpha} - {K_{XY}^{xy}}^T ((1-\alpha) K_X + \alpha K_Y)^{-1} {K_{XY}^{xy}})}^{\frac{-(1-\beta)}{2(1-\alpha)}} 
\end{split}
\label{eq:sm2TGPcost}
\end{equation}
where ${K_{xy}}^{\alpha} = (1-\alpha) k_X(x,x) + \alpha k_Y(y,y)$, $K_{XY}^{xy} = (1-\alpha) K_X^x+  \alpha K_Y^y$. Since $|K_{X\cup x}|^\frac{1-\beta}{2}$,  ${|K_{Y}|}^\frac{\alpha (1-\beta)}{(1-\alpha)}$ , and $|(1-\alpha) K_{X} + \alpha {K_{Y}}|^\frac{-(1-\beta)}{2 (1-\alpha)}$  are  multiplicative positive constants that do not depend on $y$, they can be dropped from the cost function. Also,  $- \frac{1}{\beta-1}$ is an additive constant that can be ignored under optimization. After ignoring these multiplicative positive constants and the added constant, the improved SMTGP cost function reduces to

\begin{equation}
\small
\begin{split}
 L'_{\alpha, \beta}(p(X,x) : p(Y,y)) =  &\frac{1}{\beta-1} \Big[ {(k_Y(y,y) - {K_Y^y}^T {K_Y}^{-1} {K_Y^y})^\frac{\alpha(1 -\beta)}{2(1-\alpha)}} \cdot \\
&({{{K_{xy}}^{\alpha} - {K_{XY}^{xy}}^T ((1-\alpha) K_X + \alpha K_Y)^{-1} {K_{XY}^{xy}})}^{\frac{-(1-\beta)}{2(1-\alpha)}} } \Big]
\end{split}
\label{eq:sm2tgp1}
\end{equation}
In contrast to \ignore{$L_{\alpha, \beta}(p(X,x) : p(Y,y))$}$L_{\alpha, \beta}$ in Equation~\ref{eq:lsmd1}, \ignore{$L'_{\alpha, \beta}(p(X,x) : p(Y,y))$} $L'_{\alpha, \beta}$ does not involve a determinant of a large matrix. Hence, we predict the output $y$ by  directly\footnote{There is no need to optimize over the logarithm of \ignore{$L'_{\alpha, \beta}(p(X,x):p(Y,y))$}$L'_{\alpha, \beta}$ because there is no numerical stability problem} minimizing \ignore{$L'_{\alpha, \beta}(p(X,x) : p(Y,y))$}$L'_{\alpha, \beta}$ in Equation~\ref{eq:sm2tgp1}. Since the cost function has two factors that does depend on $y$, we follow the rule that if $g(y) = c \cdot f(y) \cdot r(y)$ where $c$ is a constant,  $f(y)$ and  $r(y)$ are functions, then $\frac{\partial g(y)}{\partial y} =   c \cdot ( \frac{\partial f(y)}{\partial y} r(y) + f(y)\frac{\partial r(y)}{\partial y} )$ %$g'(y) = c \cdot ( f'(y) r(y)+f(y)r'(y))$,
 which interprets the two terms of the derived gradient below, where $f(y) = {(k_Y(y,y) - {K_Y^y}^T {K_Y}^{-1} {K_Y^y})^\frac{\alpha (1 -\beta)}{2 (1- \alpha)}}$, $r(y) = ({{K_{xy}}^{\alpha} - {K_{XY}^{xy}}^T ((1-\alpha) K_X + \alpha K_Y)^{-1} {K_{XY}^{xy}})}^{\frac{-(1-\beta)}{2(1-\alpha)}}$, $c = \frac{1}{\beta -1}$\ignore{Following matrix calculus, the analytic gradient could be computed as follows.}
\begin{equation}
\small
\begin{split}
\frac{\partial L'(\alpha,\beta)}{\partial y^{(d)}}  (p(X,x) : p(Y,y))=&  \frac{1}{\beta-1} \Big[ {\frac{\alpha (1 -\beta)}{2 (1-\alpha)} (k_Y(y,y) - {K_Y^y}^T {K_Y}^{-1} {K_Y^y})^{\frac{\alpha (1 -\beta)}{2 (1-\alpha)} -1}} \cdot \\& {(\frac{\partial k_Y(y,y)}{\partial y^{(d)}} - 2 \cdot {K_Y^y}^T {K_Y}^{-1} \frac{\partial K_Y^y}{\partial y^{(d)}})}  \cdot\\
&({{{K_{xy}}^{\alpha} - {K_{XY}^{xy}}^T ((1-\alpha) K_X + \alpha K_Y)^{-1} {K_{XY}^{xy}})}^{\frac{-(1-\beta)}{2(1-\alpha)}} } +\\
&{(k_Y(y,y) - {K_Y^y}^T {K_Y}^{-1} {K_Y^y})^\frac{\alpha (1 -\beta)}{2 (1-\alpha)}} \cdot \frac{-(1-\beta)}{2(1-\alpha)}\\ & ( {{{K_{xy}}^{\alpha} - {K_{XY}^{xy}}^T ((1-\alpha) K_X + \alpha K_Y)^{-1} {K_{XY}^{xy}})}^{\frac{-(1-\beta)}{2(1-\alpha)} -1}  }\cdot \\
&(\alpha \frac{\partial  k_Y(y,y)}{\partial y^{(d)}} - 2 \cdot {K_{XY}^{xy}}^T ((1-\alpha) K_X + \alpha K_Y)^{-1} \cdot \alpha \frac{\partial K_Y^y}{\partial y^{(d)}}) \Big]
\end{split}
\label{eqgradsmd11}
\end{equation}

The computational complexity of the cost function in Equation~\ref{eq:sm2tgp1} and the gradient in Equation~\ref{eqgradsmd11}  is quadratic at test time  (i.e. $O(N^2)$) on number of the training data.  Since ${K_Y}^{-1}$ and $(\alpha K_X + (1-\alpha) K_Y)^{-1}$ depend only on the training points, they are precomputed in the training time. Hence, our hypothesis, about the quadratic computational complexity of improved SMTGP prediction function and gradient, is true since the remaining computations are $O(N^2)$. This indicates the advantage of using our closed-form expression for SM divergence in lemma~\ref{lemma11} against the closed-form proposed in~\citep{SM:2012} with cubic complexity. However, both expression are equivalent, it is straight forward to compute the gradient in quadratic complexity from $D'(\alpha,\beta)$ expression. 
%On the other hand, it is not obvious from $D(\alpha,\beta)$ expression how to compute  the gradient in a quadratic complexity. 
\ignore{The derivation steps of $\frac{\partial L'(\alpha,\beta)}{\partial y^{(d)}} $ are detailed in Appendix C.}

\subsection{Advantage of $D'_{\alpha,\beta}(\mathcal{N}_p, \mathcal{N}_q)$ against $D_{\alpha,\beta}(\mathcal{N}_p, \mathcal{N}_q)$ out of SMTGP context} 
\label{ss:ncfa}

\ignore{In contrast to $D_{\alpha,\beta}(\mathcal{N}_p, \mathcal{N}_q)$, $D'_{\alpha,\beta}(\mathcal{N}_p, \mathcal{N}_q)$ in lemma~\ref{lemma11} could be computed similarly in $N^3$ operations, if $\Delta \mu =0$, required to compute the determinants of $\Sigma_p$, $\Sigma_q$, and $\alpha \Sigma_q +(1-\alpha) \Sigma_p$ by Cholesky decomposition. This } The previous subsection shows that the computational complexity of SMTGP prediction was decreased significantly using our $D'_{\alpha,\beta}$  at test time to be quadratic, compared to cubic complexity for  $D_{\alpha,\beta}$. Out of the TGP\ignore{using the SMTGP} context,  we show here another general advantage of using our proposed closed-form expression to generally compute SM-divergence between two Gaussian distributions $\mathcal{N}_p$ and $\mathcal{N}_q$. 
\ignore{It is worth mentioning that } $D'_{\alpha,\beta}(\mathcal{N}_p, \mathcal{N}_q)$ 
is $1.67$ times faster to compute than $D_{\alpha,\beta}(\mathcal{N}_p, \mathcal{N}_q)$ under $\Delta \mu =0$ condition. This is since $D'_{\alpha,\beta}(\mathcal{N}_p, \mathcal{N}_q)$ needs $N^3$ operations which is much less than $5 N^3/3$   operations needed to compute $D'_{\alpha,\beta}(\mathcal{N}_p, \mathcal{N}_q)$(\ie requires less matrix operations);  see Appendix C for the proof.\ignore{
 3 determinant computations  and 2 matrix inversions, if $\Delta \mu = 0$. However, equation~\ref{eq:ncfe}. requires only 3 matrix determinant computations under the same condition. This simplification could be used to efficiently compute SM divergence between two Gaussian Distributions, out of the context of TGPs. 
} We conclude this section by a general form of Lemma~\ref{lemma11} in Equation~\ref{eq:Normsm2Gen}, where $\Delta \mu \neq 0$. This equation was achieved by refactorizing the exponential term and using matrix identities.

 \begin{equation}
\small
\begin{split}
D'_{\alpha, \beta}(\mathcal{N}_p:\mathcal{N}_q) =  \frac{1}{\beta-1} \Bigg[ &\Big(\frac{|\Sigma_p| ^{1-\alpha}|{\Sigma_q}|^\alpha}{  | \alpha \Sigma_q +(1-\alpha) { \Sigma_p} |}\Big)^{\frac{(1-\beta)}{2 (1-\alpha)}}\cdot \\ & e^{-\frac{\alpha (1-\beta)}{2}   \Delta\mu^T \Sigma_q (\alpha {\Sigma_q} +(1-\alpha) {\Sigma_q})^{-1} \Sigma_p \Delta\mu } -1   \Bigg]
\end{split}
\label{eq:Normsm2Gen}
\end{equation}

\ignore{
 \textbf{Note on Lemma~\ref{lemma11}:} }
 \ignore{In addition to the $N^3$ operations needed to compute  $D'_{\alpha,\beta}(\mathcal{N}_p, \mathcal{N}_q)$ under $\Delta \mu \neq 0$, $\frac{N^3}{3}$ operations are needed to compute $(\alpha {\Sigma_q} +(1-\alpha) {\Sigma_q})^{-1}$ \footnote{Additional $2 \cdot {O(N^{2.33})}$ for 2 matrix multiplications are ignored}. So, total of $\frac{4 N^3}{3}$ operations are needed if $\Delta \mu \neq 0$,}
 
\noindent In case $\Delta \mu \neq 0$,  $D'_{\alpha,\beta}(\mathcal{N}_p, \mathcal{N}_q)$ is $1.5$ times faster than computing  $D_{\alpha,\beta}(\mathcal{N}_p, \mathcal{N}_q)$. This is since $D'_{\alpha,\beta}(\mathcal{N}_p, \mathcal{N}_q)$ needs $4 N^3 /3$ operations in this case which is less than $2 N^3$  operations needed to compute $D'_{\alpha,\beta}(\mathcal{N}_p, \mathcal{N}_q)$ under $\Delta \mu \neq 0$; see Appendix C. This indicates that the  simplifications, we provided in this work, could be used to generally speedup the computation of SM divergence between two Gaussian Distributions, beyond the context of TGPs. \ignore{However, under the  SMTGP cost function, the computational complexity was decreased significantly in test time to be quadratic, as shown in the previous subsection. }

%% file: eiganalysis3.tex
In order to understand the role of $\alpha$ and $\beta$ parameters of SMTGP, we performed an eigen analysis of the cost function in Equation~\ref{eq:sm2TGPcost}.  Generally speaking, the basic notion of TGP prediction, is to extend the dimensionality of the divergence measure from $N$ training examples to $N+1$ examples, which involves the test point $x$ and the unknown output $y$. Hence, we start by discussing the extension of a general Gaussian Process from $K_{Z}$  (e.g. $K_X$ and $K_Y$) to $K_{Z \cup z}$ (e.g. $K_{X \cup x}$ and $K_{Y \cup y}$), where $Z$ is any domain and $z$ is the point that extends $K_Z$ to $K_{Z \cup z}$, detailed in subsection~\ref{sec:ea}. Based on this discussion, we will derive two lemmas to address some properties of the SMTGP prediction in Subsection~\ref{sec:tgpcostanalysis}, which will lead to a probabilistic interpretation that we provide in subsection~\ref{sec:pi}.

\subsection{A Gaussian Process from $N$ to $N+1$ points}
\label{sec:ea}
In this section, we will use a superscript to disambiguate between the kernel matrix of size $N$ and $N+1$, i.e. $K^N$ and $K^{N+1}$.
Let $f(\mathbf{z}) = \mathcal{GP}(m({z})=0,$ $k(\mathbf{z}, \mathbf{z}'))$ be a  Gaussian process on an arbitrary domain $Z$. Let $GP^N =  \mathcal{N}({0}, K^N)$ be the marginalization of the given Gaussian process over the $N$ training points (i.e. $\{{z}_i\}, i = 1:N$). Let $GP^{N+1} = \mathcal{N}(0, K^{N+1})$ be the extension of the $GP^N$ be the marginalization of $f({z})$ over $N+1$ points after adding the ${N+1}^{th}$ point (i.e. ${z}$)\footnote{This is linked to the extending $p(X)$  to $p(X,x)$ and $p(Y)$ to $p(Y,y)$  by $x$ and $y$ respectively}. 
The kernel matrix $K^{N+1}$ is written in terms of $K^{N}$ as follows
\begin{equation}
\small
  K^{N+1} = \begin{bmatrix}
  K^N  & v \\
  v^T & k(z,z)
  \end{bmatrix}
\end{equation}
where $v = [k(z,z_1) \cdots k(z,z_N)]^T$. 
The matrix determinant of $K^{N+1}$ is related to $K^{N}$ by
 \begin{equation}
 \small
 \begin{split}
|K^{N+1}| = &  \eta \cdot |K^{N}| , \,\,\,\,\,\,\,\,\,\,\,\,
\eta  =  k(z,z) - v^T (K^{N})^{-1}  v.
\end{split}
\label{eq:eta}
\end{equation} 
Since multivariate Gaussian distribution is a special case of the elliptical distributions, the eigen values of any covariance matrix (e.g. $K^{N}, K^{N+1}$ )  are interpreted  as variance of the distribution in the direction of the corresponding eigen vectors. Hence, the determinant of the matrix (e.g. $|K^{N}|, |K^{N+1}|$) generalizes the notion of the variance in multiple dimensions as the volume of  this elliptical distribution, which is oriented by the eigen vectors. From this notion, one could interpret  $\eta$ as the ratio by which the variance (uncertainty) of the marginalized Gaussian process is scaled, introduced by the new data point $z$. Looking closely at $\eta$,  we can  notice

(1) $0 \le \eta \le k(z,z)$,  since $|K^{N}|>0$, $|K^{N+1}|>0$, and $v^T (K^{N})^{-1}  v \geq 0$.

(2)  In the case of the regularized Gaussian kernel, we used in our work,  $k(z,z) = 1 +\lambda$, and hence $0 \le \eta \le 1+\lambda$

(3) $\eta$  decreases as the new data point get closer to the $N$ points.  This situation makes $v$  highly correlated with the eigen vectors of small eigen values of $K^{N}$, since the term $v^T (K^{N})^{-1}  v $ is maximized as $v$ points to the smallest principal component of $K^{N}$ (i.e. the direction of the maximum certainty). Hence, $\eta$ is an uncertainty measure, which is minimized as the new data point $z$  produces a vector $v$, that maximizes the certainty of the data under $\mathcal{N}(0, K^N)$, which could be thought as a measurement proportional to $1/ p(z| z_{1}:z_{N})$. Computing  $\eta$ on the input space $X$ makes it equivalent to the predictive variance of Gaussian Process Regression (GPR) prediction~\citep{Rasmussen:2005} (Chapter 2), which depends only on the input space. However, we are discussing $\eta$ as an uncertainty extension from $N$ to $N+1$ on an arbitrary domain, which is beneficial for SMTGP analysis that follows.
\ignore{
This interpretation of $\eta_{\mathbf{z}_{N+1}}$ is the central notion, upon which we build our analysis of Twin Gaussian Processes in the next subsection.
}

\subsection{TGP Cost Function Analysis}
\label{sec:tgpcostanalysis}
We start by the  optimization function of the SMTGP prediction, defined as  
\begin{equation}
\small
\begin{split}
\hat{y}&(\alpha,\beta) =  \underset{y}{\operatorname{argmin}} \Big[ D'_{\alpha, \beta}(GP_{X\cup x} : GP_{Y\cup y}) =  \frac{1}{\beta-1} \Big(\big(\frac{|K_{X\cup x}|^{1-\alpha} {|K_{Y \cup y}|}^{\alpha}}{|(1-\alpha) K_{X\cup x} + \alpha {K_{Y \cup y}}|})^{\frac{1-\beta}{2(1-\alpha)}} -1)\Big) \Big]  
\end{split}
\label{eq:asma}
\end{equation}
where $D'_{\alpha, \beta}(\cdot,\cdot)$ is as defined in Equation~\ref{eq:sm2TGPcost}.
As detailed in Section~\ref{sec:smTGP}, SM divergence,  involves the determinant of three matrices of size $N+1 \times N+1$, namely $K_{X\cup x}$,  $K_{Y\cup y}$, and $ \alpha K_{Y \cup y} +(1-\alpha) K_{X\cup x} $. Hence,  We have three uncertainty extensions from  $N$ to $N+1$, as follows
\begin{equation}
\small
\begin{split}
|K_{X\cup x}|= \eta_x  \cdot |K_X| , \,\,\,\,\,\,\,\,\,\, |K_{Y\cup y}| = &\eta_y \cdot |K_Y| \\
 |\alpha K_{Y \cup y} +(1-\alpha)  K_{X\cup x}|=  \eta_{x,y}(\alpha) \cdot |&\alpha K_Y +(1-\alpha) K_X| , \\ 
\text{where } \eta_{x,y}(\alpha) = \alpha K_Y(y,y) +(1-\alpha) K_X(x,x) - v_{xy}(\alpha)^T &(\alpha K_Y +(1-\alpha) K_X)^{-1} v_{xy}(\alpha), \\
v_{xy}(\alpha) = \alpha K_Y^y +(1-\alpha) K_X^x &\\
% \eta_{x} = k_X(x,x) - v_{x}(\alpha)^T K_X^{-1} v_{x}, v_{x}=  K_X^x & \\
%  \eta_{y} = k_Y(y,y) - v_{y}(\alpha)^T K_Y^{-1} v_{y}, v_{y}=  K_Y^y & \\
\end{split}
\label{eq:3ext}
\end{equation}
It might not be straightforward to think about $ \alpha K_{Y \cup y} +(1-\alpha) { K_{X\cup x}} $ within TGP formulation as a kernel matrix defined on $X \times Y$ space in Equation~\ref{eq:3ext}. This gives an interpretation of the constraint that $0 \leq \alpha \leq1$ in Equation~\ref{eqcfsm1} and~\ref{eq:Normsm2}. Since  $\alpha k_Y(y_i, y_j) + (1-\alpha) k_X(x_i, x_j)$ is a weighted sum of valid kernels with positive weights, then $ \alpha K_{Y \cup y} +(1-\alpha) { K_{X\cup x}} $ is a valid kernel matrix  on $X \times Y$ space. From Equation~\ref{eq:asma} and~\ref{eq:3ext}, we derived with the following two Lemmas. 
\begin{lem}
Under SMTGP, $\varphi_{\alpha}(x,y) = \frac{  \eta_x^{1-\alpha}  \eta_y^{\alpha}}{\eta_{x,y}(\alpha)}\le (\int_{-\infty}^{\infty}{p_{Y}(t)^\alpha p_{X}(t)^{1-\alpha} dt} )^{-2}$, $(\int_{-\infty}^{\infty}{p_{Y}(t)^\alpha p_{X}(t)^{1-\alpha} dt} )^{-2} =  \frac{  | \alpha K_{Y } +(1-\alpha) { K_{X}} | }{|K_{X}| ^{1-\alpha} |{K_{Y}}|^\alpha }  \ \ge 1$
\label{lemma1}
\end{lem}

\begin{proof}
Directly from the definition of SM TGP in Equation~\ref{eq:asma} and~\ref{eq:3ext}, SM TGP cost function could be written as, 
{\small
\begin{equation}
\begin{split}
&\hat{y}(\alpha,\beta) =   \underset{y}{\operatorname{argmin}} \Bigg [D'_{\alpha, \beta}(p(X,x) : p(Y,y))  = \frac{1}{\beta-1} \Bigg( \Big(\frac{|K_{X}| ^{1-\alpha} . \eta_x^{1-\alpha} .|{K_{Y}}|^\alpha . \eta_y^{\alpha}}{  | \alpha K_{Y } +(1-\alpha) { K_{X}} | .  \eta_{x,y}(\alpha)}\Big)^{\frac{(1-\beta)}{2 (1-\alpha)}} -1   \Bigg)\Bigg]
\end{split}
\label{eq:asm}
\end{equation}}
Comparing Equation~\ref{eq:smgen} to Equation~\ref{eq:asm}, then 

\begin{equation*}
\Big(\frac{|K_{X}| ^{1-\alpha}| . \eta_x^{1-\alpha} .{K_{Y}}|^\alpha . \eta_y^{\alpha}}{  | \alpha K_{Y } +(1-\alpha) { K_{X}} | .  \eta_{x,y}(\alpha)}\Big)^\frac{1}{2}    = \int_{-\infty}^{\infty}{p_{X,x}(t)^\alpha p_{Y,y}(t)^{1-\alpha} dt} \le1,
\end{equation*}
and since  $\eta_{x,y}(\alpha)>0$ and  $\eta_y>0$, then 

\begin{equation*}\varphi_{\alpha}(x,y) = \frac{\eta_x^{1-\alpha}  \eta_y^{\alpha}}{ \eta_{x,y}(\alpha)} \le \frac{  | \alpha K_{Y } +(1-\alpha) { K_{X}} | }{|K_{X}| ^{1-\alpha} |{K_{Y}}|^\alpha }, \end{equation*}
and since  $\int_{-\infty}^{\infty}{p_X(t)^\alpha p_Y(t)^{1-\alpha} dt} = \big(\frac{|K_{X}| ^{1-\alpha} |{K_{Y}}|^\alpha }{  | \alpha K_{Y } +(1-\alpha) { K_{X}} | } \big)^{\frac{1}{2}}$ and 

$\int_{-\infty}^{\infty}{p_X(t)^\alpha p_Y(t)^{1-\alpha}} dt \le 1$, then $\frac{|K_{X}| ^{1-\alpha} |{K_{Y}}|^\alpha }{  | \alpha K_{Y } +(1-\alpha) { K_{X}} | } \le 1$.
\end{proof}

\begin{lem}
Under SMTGP and  $0 < \alpha < 1$, $\hat{y}(\alpha, \beta)$ maximizes $ \varphi_{\alpha}(x,y) = \frac{ \eta_x^{1-\alpha} . \eta_y^\alpha}{\eta_{x,y}(\alpha)} \le \frac{  | \alpha K_{Y } +(1-\alpha) { K_{X}} | }{|K_{X}| ^{1-\alpha} |{K_{Y}}|^\alpha }$ and it does not depend on $\beta$ theoretically. 
\label{lemma2}
\end{lem}

\begin{proof}
\ignore{
The proof of this lemma depends on the claim that the cost functions  $D'_{\alpha, 1-\tau}(p(X,x) : p(Y,y))$  and  $D'_{\alpha, 1+\zeta}(p(X,x) : p(Y,y))$  are equivalent. We start by denoting $Z_{\alpha}(y) = \frac{|K_{X\cup x}| ^{1-\alpha}|{K_{Y \cup y}}|^\alpha}{  | \alpha K_{Y \cup y} +(1-\alpha) { K_{X\cup x}} |}= (\int_{-\infty}^{\infty}{p_{X,x}(t)^\alpha p_{Y,y}(t)^{1-\alpha} dt} )^2 \le  1$. From this notation,   $D'_{\alpha, 1-\tau}(p(X,x) : p(Y,y))$  and  $D'_{\alpha, 1+\zeta}(p(X,x) : p(Y,y))$ could be re-written as}
We start by the claim that $\hat{y}(\alpha, 1-\tau) = \hat{y}(\alpha, 1+\zeta)$, $0 < \alpha < 1$, $\tau>0$, $\zeta>0$ and both predictions are achieved by maximizing  $(\int_{-\infty}^{\infty}{p_{(Y,y)}(t)^\alpha p_{(X,x)}(t)^{1-\alpha} dt} )^2$  $=  \frac{|K_{X}| ^{1-\alpha} |{K_{Y}}|^\alpha  \eta_x^{1-\alpha}  \eta_y^\alpha}{  | \alpha K_{Y } +(1-\alpha) { K_{X}} |  \eta_{x,y}{\alpha}} \le 1$,  which is $\propto  \varphi_{\alpha}(x,y)$, where  $p_{(X,x)} = \mathcal{N}(0, K_{X \cup x})$ and $p{(Y,y)} = \mathcal{N}(0, K_{Y \cup y})$, $p{(X)} = \mathcal{N}(0, K_{X})$ and $p{(Y)} = \mathcal{N}(0, K_{Y})$. This claim indicates that the cost functions  $D'_{\alpha, 1-\tau}(p(X,x) : p(Y,y))$  and  $D'_{\alpha, 1+\zeta}(p(X,x) : p(Y,y))$  are equivalent. 

Let us introduce $Z_{\alpha}(y) = \frac{|K_{X\cup x}| ^{1-\alpha}|{K_{Y \cup y}}|^\alpha}{  | \alpha K_{Y \cup y} +(1-\alpha) { K_{X\cup x}} |}= (\int_{-\infty}^{\infty}{p_{X,x}(t)^\alpha p_{Y,y}(t)^{1-\alpha} dt} )^2 \le  1$. From this notation,   $D'_{\alpha, 1-\tau}$ $(p(X,x) : p(Y,y))$  and  $D'_{\alpha, 1+\zeta}(p(X,x) : p(Y,y))$ could be re-written as

{\small
\begin{equation}
\small{
\begin{split}
&D_{\alpha, 1-\tau}(p(X,x) : p(Y,y)) = \frac{1}{-\tau} \Bigg[ \Big(Z_{\alpha}(y)\Big)^{\frac{\tau}{2 (1-\alpha)}} -1   \Bigg],\\
%&D_{\alpha, 1+\tau}(p(X,x) : GP_{Y\cup y})) = \frac{1}{\tau} \Bigg( \Big(Z_{\alpha}(y)\Big)^{\frac{-\tau}{2 (1-\alpha)}}  -1  \Bigg]\\
%&D_{\alpha, 1+\tau}(p(X,x) : GP_{Y\cup y})) = \frac{1}{\tau} \Bigg( \frac{1}{-\tau D_{\alpha, 1-\tau}(GP_{X\cup x} : GP_{Y\cup y}) +1} -1  \Bigg]\\
&D_{\alpha, 1+\zeta}(p(X,x) : p(Y,y)) = \frac{1}{\zeta} \Bigg[ \Big(Z_{\alpha}(y)\Big)^{\frac{-\zeta}{2 (1-\alpha)}}  -1  \Bigg]\\
\end{split}
}
\label{eq:asmcost}
\end{equation}}
From Equation~\ref{eq:asmcost} and under the assumption that $0<\alpha<1$ and $Z_{\alpha}(y)\le1$, then $D_{\alpha, 1-\tau}(p(X,x) : p(Y,y))\ge0, D_{\alpha, 1+\zeta}(p(X,x) : p(Y,y)) \ge0$. Both are clearly  minimized as $Z_{\alpha}(y)$ approaches 1 (i.e. maximized,   since $Z_{\alpha}(y)\le1$).  Comparing  Equation~\ref{eq:asm} and~\ref{eq:asmcost}, $Z_{\alpha}(y) =  (\int_{-\infty}^{\infty}{p_{X,x}(t)^\alpha p_{Y,y}(t)^{1-\alpha} dt} )^2 =  \frac{|K_{X\cup x}| ^{1-\alpha}|{K_{Y \cup y}}|^\alpha}{  | \alpha K_{Y \cup y} +(1-\alpha) { K_{X\cup x}} |}=   \frac{|K_{X}| ^{1-\alpha}| .  \eta_x^{1-\alpha} .{K_{Y}}|^\alpha . \eta_y^{\alpha}}{  | \alpha K_{Y } +(1-\alpha) { K_{X}} | .  \eta_{x,y}(\alpha)}\  \propto \varphi_{\alpha}(x,y)= \frac{  \eta_x^{1-\alpha}  \eta_y^\alpha}{\eta_{x,y}(\alpha)}$, since $| \alpha K_{Y } +(1-\alpha) { K_{X}} | $,  $|{ K_{X}} | $, and  $| K_{Y } | $ do not depend on the predicted output $\hat{y}$.  This indicates that SMTGP optimization function is inversely proportional to $\varphi_{\alpha}(x,y)$, we upper-bounded in Lemma~\ref{lemma1}. Hence, it is not hard to see that $\zeta$ and $\tau$ controls whether to maximize $\varphi_{\alpha}(x,y)^\frac{\tau}{1-\alpha}$ or maximize $-\varphi_{\alpha}(x,y)^\frac{-\zeta}{1-\alpha}$, which  are equivalent. This directly leads to that $\hat{y}(\alpha, \beta)$ maximizes $\frac{ \eta_x^{1-\alpha} . \eta_y^\alpha}{\eta_{x,y}(\alpha)} \le \frac{  | \alpha K_{Y } +(1-\alpha) { K_{X}} | }{|K_{X}| ^{1-\alpha} |{K_{Y}}|^\alpha }$ and it does not depend on $\beta$ theoretically. 
\end{proof}

The proof of Lemma~\ref{lemma2} shows the relationship between $\varphi_{\alpha}(x,y)$ and SM divergence through a derivation that starts from SMTGP cost function. From lemma~\ref{lemma1} and~\ref{lemma2}, the term $\frac{|K_{X}| ^{1-\alpha} |{K_{Y}}|^\alpha }{  | \alpha K_{Y } +(1-\alpha) { K_{X}} | } \le 1$ represents an agreement function between $p(X)$  and $p(Y)$. Similarly, $\frac{|K_{X\cup x}| ^{1-\alpha} |{K_{Y \cup y}}|^\alpha }{  | \alpha K_{Y\cup y} +(1-\alpha) { K_{X \cup y}} | } = \frac{|K_{X}| ^{1-\alpha} |{K_{Y}}|^\alpha  \eta_x^{1-\alpha}  \eta_y^\alpha}{  | \alpha K_{Y } +(1-\alpha) { K_{X}} |  \eta_{x,y}{\alpha}} \le 1$ is an agreement function between the extended distributions $p(X,x)$ and $p(Y,y)$. This agreement function increases as the weighted volume of the input and the output distributions (i.e..${|K_{X\cup x}| ^{1-\alpha} |{K_{Y \cup y}}|^\alpha }$, weighted by $\alpha$) is as close as possible to the volume of the joint distribution (i.e. ${  | \alpha K_{Y\cup y} +(1-\alpha) { K_{X \cup x }} | }$). This function reaches $1$ (i.e. maximized)  when the two distributions are identical, which justifies maximizing $\varphi_{\alpha}(x,y)$ as indicated in lemma~\ref{lemma2}. From another view, maximizing $\varphi_{\alpha}(x,y)$ prefers minimizing $\eta_{x,y}(\alpha)$, which maximizes the $p((x,y)| (x_1, y_1),.. ,(x_N$ $,y_N))$, that we abbreviate as $p(x,y)$; this is motivated by  our intuition in Subsection~\ref{sec:ea}. However, SMTGP maximizes  $\varphi_{\alpha}(x,y) = \frac{  \eta_x^{1-\alpha}  \eta_y^{\alpha}}{\eta_{x,y}(\alpha)}$, this gives a probabilistic  sense for the cost function when we follow our intuition that $\eta_x \propto {1}/{p(x)}$, $\eta_y \propto {1}/{p(y)}$  and $\eta_{x,y}(\alpha) \propto {1}/{p(x,y)}$. Hence  $\varphi_{\alpha}(x,y)$ could be seen as $\frac{p(x,y)}{p(x)^{1-\alpha} p(y)^{\alpha}}$, discussed in the following subsection.  This understanding motivated us to plot the relation  between $\varphi_{\alpha}(x,y)$  and the test error on SMTGP prediction. Figure~\ref{fig:phixy} shows a clear correlation between  $\varphi_{\alpha}(x,y)$ and the prediction error. Hence, it introduces a clear motivation to study it as a certainty measure, which could be associated with each structured output prediction.
\vspace{-5mm}
\begin{figure}[h!] 
  \centering 
  \includegraphics[width=0.7\textwidth]{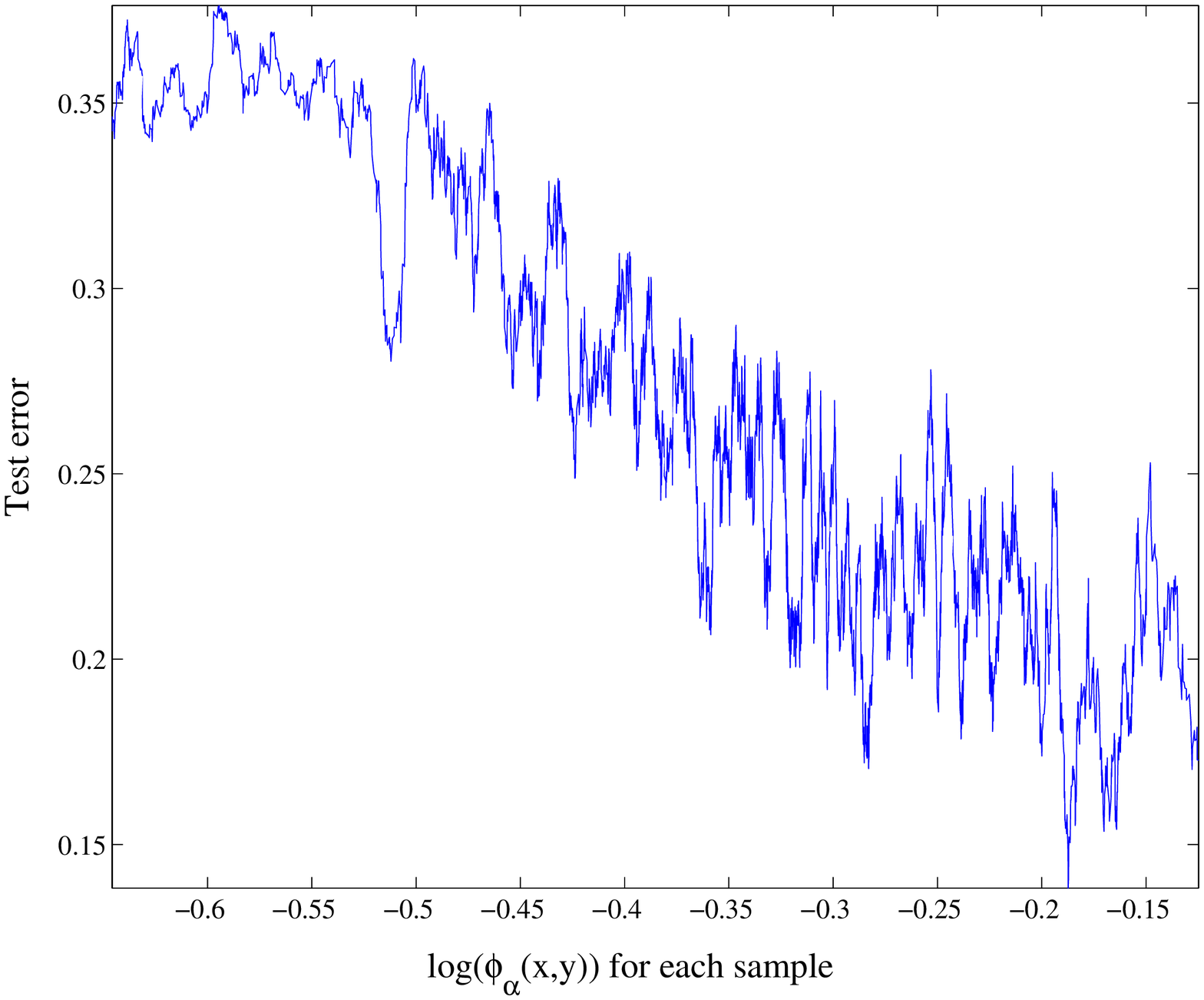}
  \caption{$log(\varphi_{\alpha}(x,y))$ against test error on USPS dataset using SMTGP, $\alpha = 0.8$, $\beta = 0.5$ }
  \label{fig:phixy}
\end{figure}
\vspace{-5mm}
\ignore{
\begin{wrapfigure}{r}{0.25\textwidth}   
  \includegraphics[width=0.25\textwidth]{phixy.eps}
  \vspace{-10mm}
  \caption{$\varphi_{\alpha}(x,y)$ against test error on USPS dataset using SMTGP, $\alpha = 0.8$, $\beta = 0.5$ }
    \vspace{-10mm}
  \label{fig:phixy}
\end{wrapfigure}
}

\ignore{
\begin{figure}[h!]
  \begin{center}
    \includegraphics[width=0.25\textwidth]{phixy.eps}
  \end{center}
  \caption{$\varphi_{\alpha}(x,y)$ against test error on USPS dataset using SMTGP, $\alpha = 0.8$, $\beta = 0.5$ }
  \label{fig:phixy}
\end{figure}
}

\subsection{Probabilistic Interpretation of Maximizing $\varphi_{\alpha}(x,y) = \frac{ \eta_x^{1-\alpha} . \eta_y^\alpha}{\eta_{x,y}(\alpha)}$}
\label{sec:pi}

As detailed in the previous subsection, one can interpret $\eta_{x,y}(\alpha) \propto {1}/{p(x,y)}$, $\eta_x \propto {1}/{p(x)}$, $\eta_y \propto {1}/{p(y)}$. Hence,$\frac{\eta_{x,y}(\alpha)}{ \eta_y} \propto p(y|x)$,  $\frac{ \eta_x}{\eta_{x,y}(\alpha)} \propto p(x|y)$. Hence, what does $\frac{ \eta_x^{1-\alpha} . \eta_y^\alpha}{\eta_{x,y}(\alpha)}$ mean?  Since $0<\alpha<1$, it is obvious that $min(\eta_x,\eta_y) <f_1(\alpha) = \eta_x^{1-\alpha}.\eta_y^{\alpha}< max(\eta_x,\eta_y)$. Figure~\ref{fig:etaalphas} shows the behavior of $f_1(\alpha)$ against $f_2(\alpha) = (1-\alpha) \cdot \eta_x +\alpha \cdot \eta_y$, which is also bounded between $min(\eta_x,\eta_y)$ and $max(\eta_x,\eta_y)$. According to this figure, $f_1(\alpha)$ behaves very similar to $f_2(\alpha)$ as $|\eta_x-\eta_y|$ approaches zero, where linear approximation is accurate. However, as $|\eta_x-\eta_y|$ gets bigger, $f_1(\alpha)$ gets biased towards $min(\eta_x, \eta_y)$ as indicated in the left column of figure~\ref{fig:etaalphas}. Hence, $\frac{\eta_{x,y}(\alpha)}{ \eta_x^{1-\alpha} . \eta_y^\alpha}$ is interpreted depending on the values of $\eta_x \propto \frac{1}{p(x)}$, $\eta_y \propto \frac{1}{p(y)}$, and  $\eta_{x,y}(\alpha) \propto \frac{1}{p(x,y)}$ as follows:
\begin{enumerate}
\item{If $\eta_x<<\eta_y$, $\frac{ \eta_x^{1-\alpha} . \eta_y^\alpha}{\eta_{x,y}(\alpha)} \widehat{\approx}$\footnote{$\widehat{\approx}$  indicates equivalence for optimization/prediction}$\frac{ \eta_x}{\eta_{x,y}(\alpha)} \propto p(y|x)$ }
\item{If $\eta_y<<\eta_x$, $\frac{ \eta_x^{1-\alpha} . \eta_y^\alpha}{\eta_{x,y}(\alpha)} \widehat{\approx} \frac{ \eta_y}{\eta_{x,y}(\alpha)} \propto p(x|y)$ }
\item{If $\eta_y\approx\eta_x$, $\frac{ \eta_x^{1-\alpha} . \eta_y^\alpha}{\eta_{x,y}(\alpha)} \widehat{\approx} \frac{ \eta_y}{\eta_{x,y}(\alpha)} \propto p(x|y) \approx \frac{ \eta_x}{\eta_{x,y}(\alpha)} \propto p(y|x)$ }; This is less likely to happen since $p(x) = p(y)$ in this case.
\item{If $|\eta_y-\eta_x|<\epsilon$, in this case $\alpha$ linearly control  $\frac{ \eta_x^{1-\alpha} . \eta_y^\alpha} {\eta_{x,y}(\alpha)} \widehat{\approx} \frac{ (1-\alpha) \eta_x + \alpha \eta_y}{\eta_{x,y}(\alpha)} \propto \big(\frac{1-\alpha}{p(y|x)} + \frac{\alpha}{p(x|y)}\big)^{-1}$ }
\end{enumerate}

Hence, SM TGP regression predicts the output of maximum certainty on $p(x,y) = \mathcal{N}(0, (1-\alpha)K_{X \cup x} +\alpha K_{Y\cup y})$, conditioned on the uncertainty extension on $p(x) = \mathcal{N}(0, K_{X\cup x})$ and $p(y) = GP(0, K_{Y\cup y} )$. The conditioning is biased towards $max(p(x), p(y))$, which gives best discrimination relative to $p(x,y)$ and hence, maximize the certainty of the prediction. In case the difference between $p(x)$ and $p(y)$ is not high, the prediction is based on a  weighted sum of $p(y|x)$ and $p(x|y)$, as shown in point 4 above.
\begin{figure}[!htb!]
\centering
\includegraphics[scale=.5]{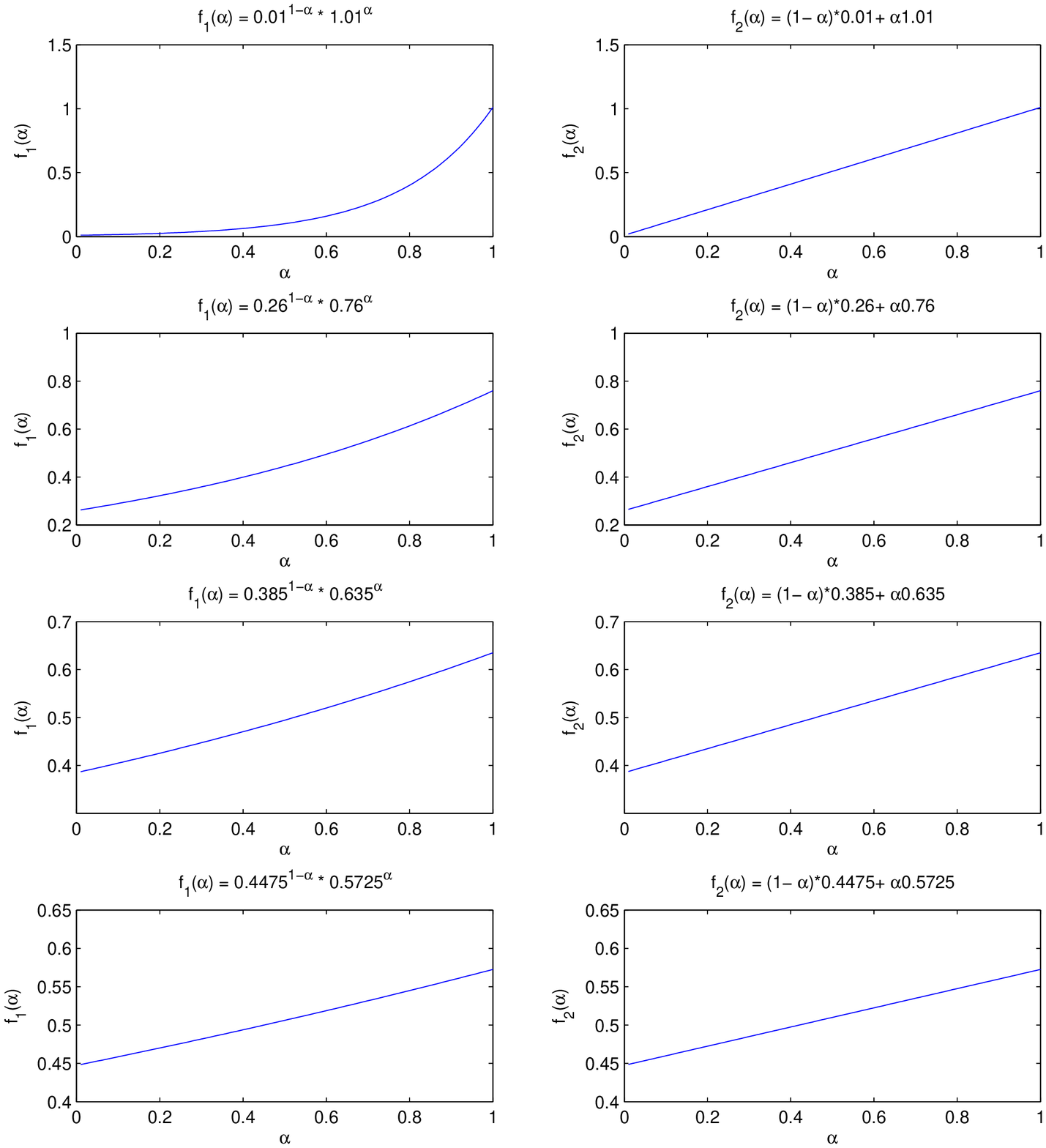}
\vspace{-15mm}
\caption{Left plot functions are in the form $f_1(\alpha) =  \eta_{D_1}^{1-\alpha}\eta_{D_2}^{\alpha}$ and their corresponding $f_2(\alpha) = (1-\alpha)\cdot \eta_{D1} +\alpha \cdot \eta_{D_2}$ are on the right;  rows indicate different values of $\eta_{D_1},\eta_{D_2} $, where $D_1$ and $D_2$ are arbitrary two domains}
\label{fig:etaalphas}
\end{figure}

%% file: experiments3.tex
\begin{sloppypar}
In this section, we evaluate SMTGP on two Toy examples, USPS dataset in an image reconstruction task, and both Poser dataset~\citep{Trig06} and HumanEva dataset~\citep{SigalBB10} for a 3D pose estimation task. It is shown in~\citep{Bo:2010,Yamada:2012}, that TGP outperforms  Kernel Regression (KR), Gaussian Process Regression (GPR), Weighted K-Nearest Neighbor regression~\citep{Rasmussen:2005}, Hilbert Schmidt independence criterion (HSIC)~\citep{HSIC05}, and Kernel Target Alignment method(KTA)~\citep{KTA01} on a Toy example, HumanEva dataset, and Poser Dataset (i.e. Pose Estimation datasets). Hence, we extended our evaluation beyond\ignore{not to cover only} pose estimation datasets. We compared our SMTGP with KLTGP and IKLTGP. IKLTGP stands for inverse KLTGP, which predicts the output by minimizing the KL divergence of the output probability distribution from the input probability distribution~\citep{Bo:2010}. The main motivation behind this comparison is that KLTGP and IKLTGP are biased to one of the distributions, and therefore the user has to choose either to use KLTGP or IKLTGP based on the problem. In contrast, SMTGP could be adapted by $\alpha$ and $\beta$ on the validation set, such that the prediction error is minimized. From this point, we denote the set of KLTGP, IKLTGP and SMTGP as \textit{TGPs}. Our presentation of the results starts by  the specification of the toy examples and the datasets in subsection~\ref{sec:exprspecs}. Then, we present our parameter settings and  how $\alpha$ and $\beta$  are selected in subsection~\ref{sec:lab}. Finally, we  show our argument on the performance on  these tasks in subsection~\ref{sec:exprdisc}.
\end{sloppypar}
\subsection{Specification of the Toy Examples and the Datasets}
\label{sec:exprspecs}

\ignore{\textbf{{Toy Example 1}}: }
\subsubsection{Toy Example 1~\citep{Bo:2010}}
The training set for the first toy problem predict a 1D output variable $y$ given a 1D control $x$ (the input). It consists of 250 values of $y$ generated uniformly in (0,1), for which $x =y+0.3 sin(2 y \pi)+ \epsilon$ is evaluated with $\epsilon$ such that $\epsilon  = N(\mu = 0, \sigma  = 0.005)$; see Figure~\ref{fig:toyex1}. Stars correspond to examples where $KNN$  regression and $GPR$ suffer from ‘boundary/discontinuous effects’ as indicated in~\citep{Bo:2010}. The \textit{TGPs} were tested with 250 equally spaced inputs $x$ in $(0,1)$. \ignore{
\noindent{\textbf{Parameter Settings and Error Measures}:} The value of $\lambda_x$ and $\lambda_y$ is $10^{-4}$ in this example. We used $2 \rho_x^2 = 5$ for the input kernel bandwidth, and $2 \rho_y^2 = 5000$ for the output kernel bandwidth. } We used the mean prediction error to measure the performance on this example.
\ignore{
\begin{figure}[h!]
 \centering
    \includegraphics[width=0.6\textwidth]{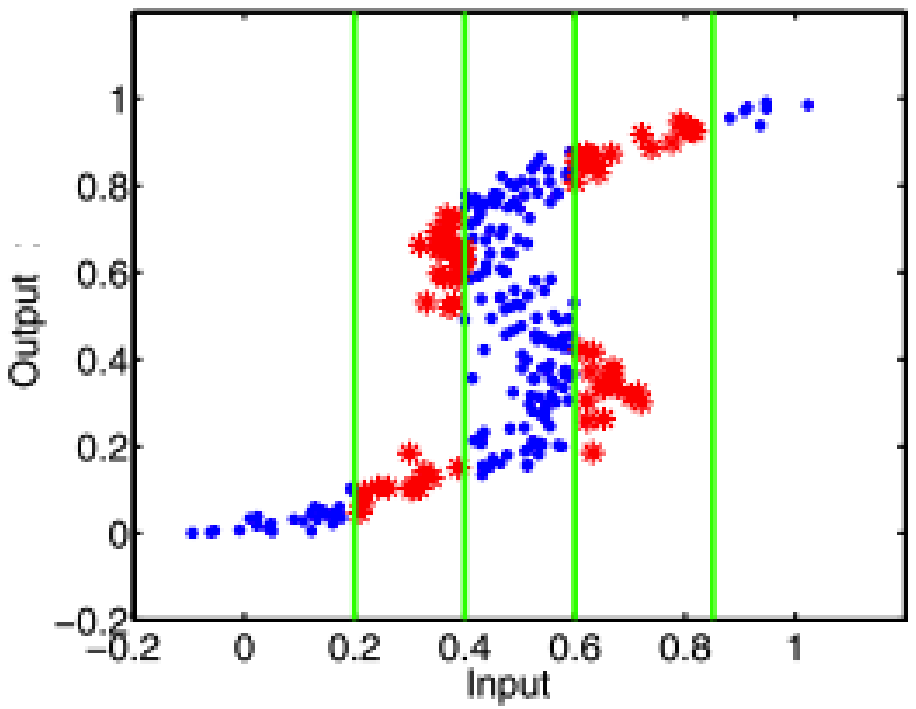}
    \caption{Toy Example 1}
      \label{fig:toyex1}
\end{figure}
}

\begin{figure}[ht!]
  \begin{minipage}[b]{0.4\linewidth}
    \centering
    \includegraphics[width=0.8\textwidth]{ToyEx.eps}
    \caption{Toy Example 1}
      \label{fig:toyex1}
  \end{minipage} 
   \begin{minipage}[b]{0.6\linewidth}
    \centering
    \includegraphics[width=1.0\textwidth]{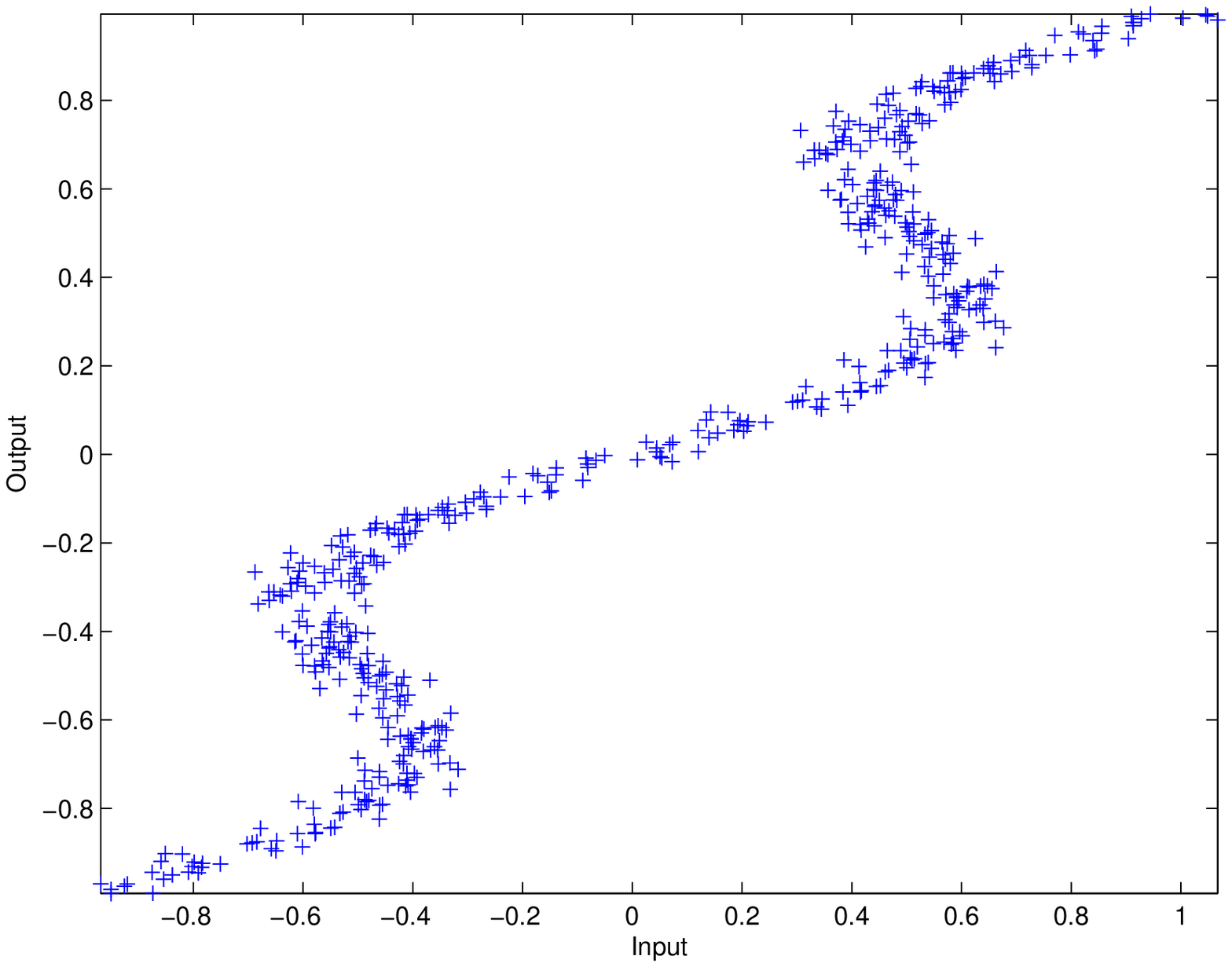}
     \caption{Toy Example 2}
     \label{figtoy2s}
  \end{minipage}
\end{figure}

\ignore{ 
\begin{figure}[h!]
  \begin{center}
    \includegraphics[width=0.3\textwidth]{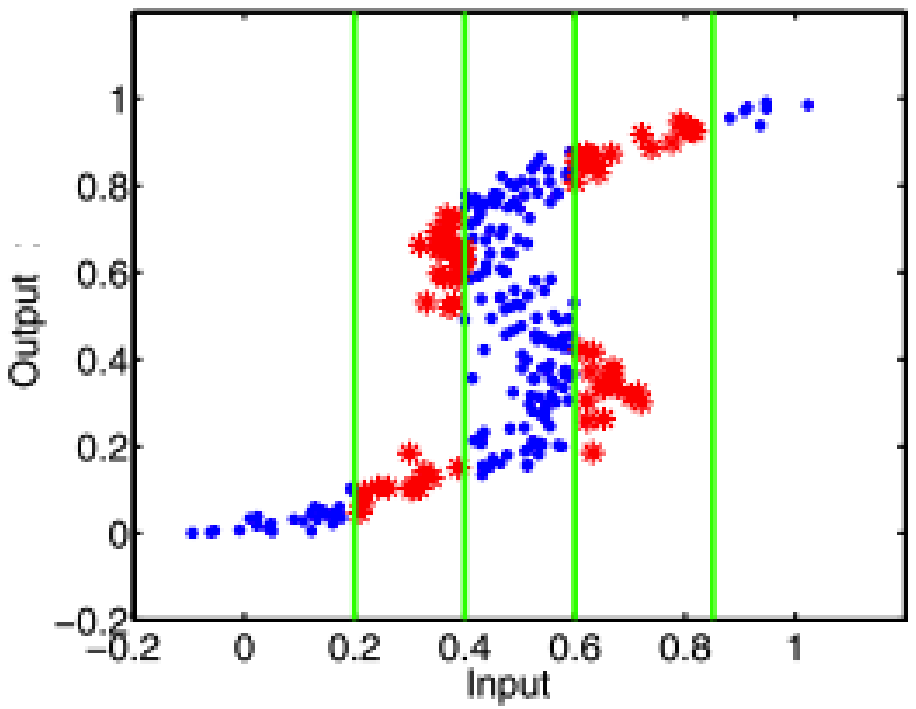}
  \end{center}
  \caption{Toy Example~\citep{Bo:2010}}
  \label{fig:toyex1}
\end{figure}}

\ignore{\textbf{{Toy Example 2}:} }

\subsubsection{Toy Example 2}
In order to introduce a more challenging situation, we generate a double $S$ shape; see Figure~\ref{figtoy2s}. Toy example 2 is constructed by concatenated two $S$ shapes, which makes the overall prediction error more challenging to reduce. In addition, we down-sampled the points by 2, such that the total number of points is the same as Toy example 1. Hence, there is less evidence on training data compared to Toy example 1. Similarly, the \textit{TGPs} were tested with 500 equally spaced inputs $x$ in $(-1,1)$. We used the same error-measure in Toy Example 1.

\ignore{
\begin{figure}[h!]
    \centering
    \includegraphics[width=0.8\textwidth]{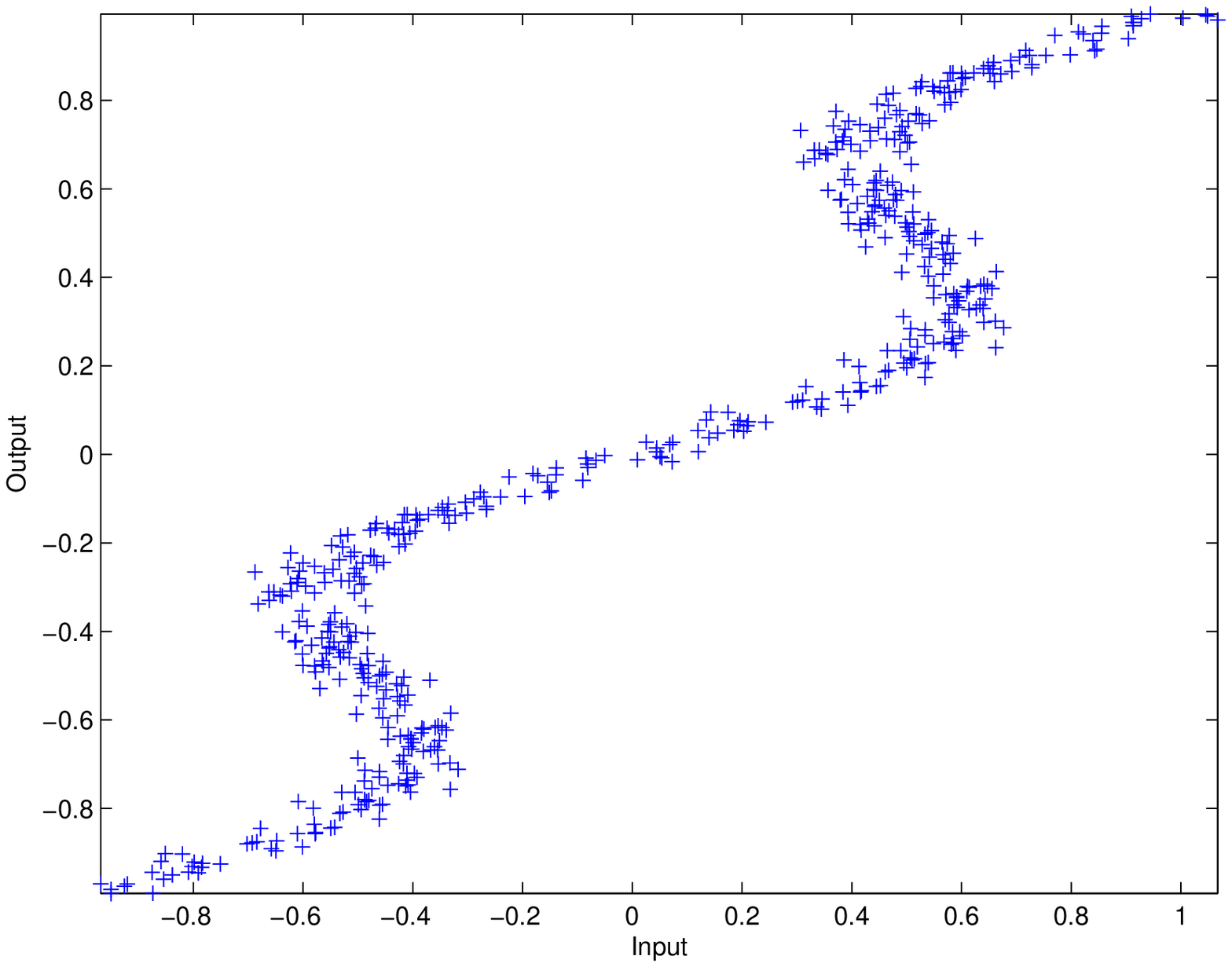}
     \caption{Toy Example 2}
     \label{figtoy2s}
\end{figure}
}
\ignore{
\noindent{\textbf{Parameter Settings and Error Measures}:} Same as Toy Example 1. 
}
\ignore{
\begin{figure}[h!]
  \begin{center}
    \includegraphics[width=1.0\textwidth]{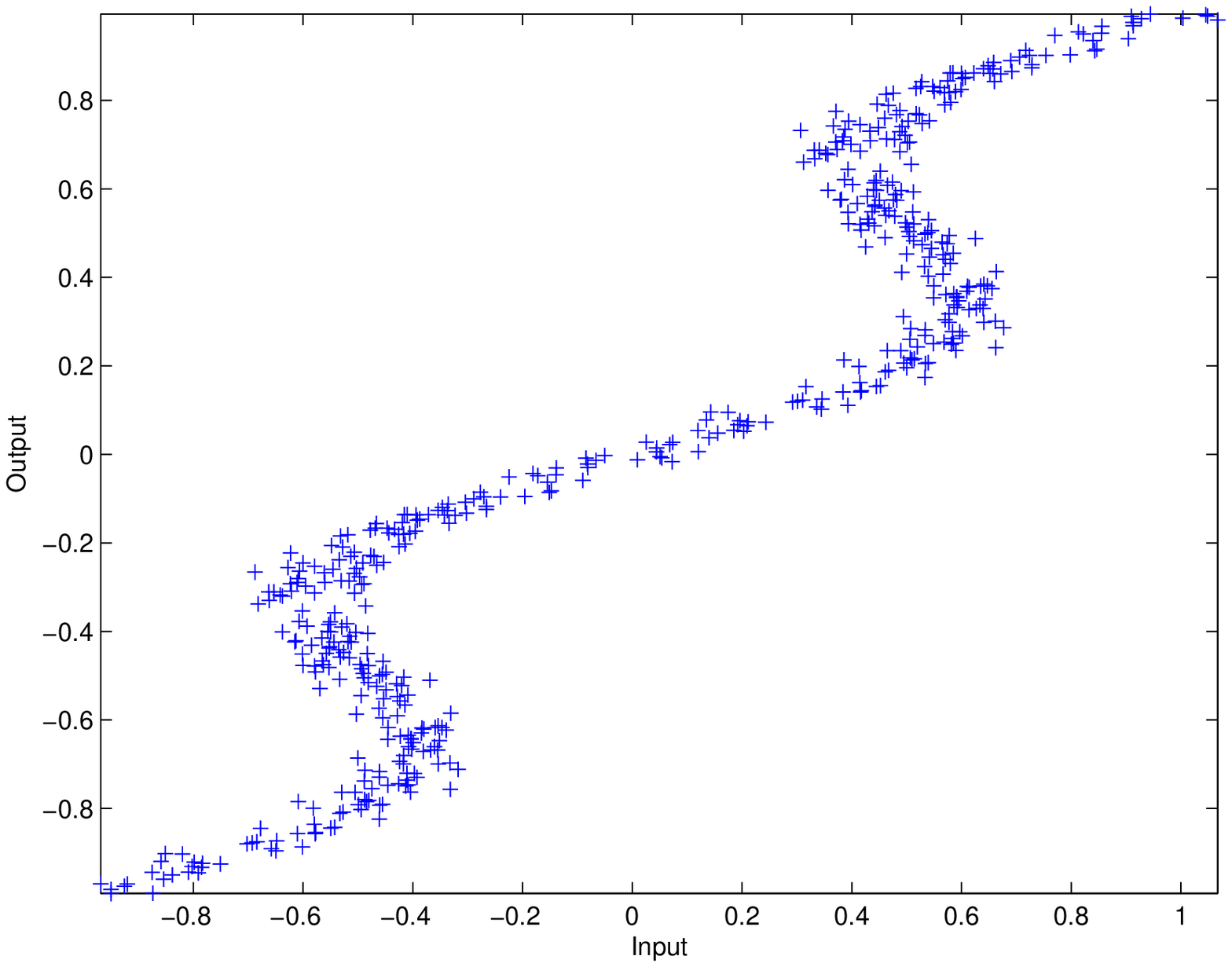}
  \end{center}
  \caption{double $S$ Toy Example}
  \label{figtoy2s}
\end{figure}}\ignore{
\begin{table*}[htbp]
  \centering
  \caption{Mean error on Toy (1,2), USPS, Poser and HumanEva Datasets}	
  	\scalebox{1.0}{
 	\begin{tabular}{c|c|c|c|c|c|}
    \toprule
          & \textbf{Toy1 Example} & \textbf{Toy2 Example} & \textbf{USPS DS} & \textbf{Poser DS} & \textbf{HumanEva DS} \\
    \midrule
    \textbf{SMTGP error} & 0.1126 & 0.110  & 0.2587  & 5.4296  & 37.59 \\
    \hline
    \textbf{KLTGP error}~\citep{Bo:2010} & 0.1160 & 0.126 & 0.2665  & 5.4296  & 37.64 \\
    \textit{\textbf{Improvement}} & \textbf{2.93\%} & \textbf{12.70\%} & 3.03\% & 0.00\% & 0.13\% \\
    \hline
    \textbf{IKLTGP error}~\citep{Bo:2010} & 0.1150 & 0.114      &  0.2683     &     6.484   & 55.622 \\
    \textit{\textbf{Improvement}} & 2.09\% &   1.3\%    & 3.71\%   &  \textbf{16.3\%}        &  \textbf{32.41\%}\\
    \bottomrule
    \end{tabular}%
    }
  \label{tab:experiments}%
\end{table*}% 
}

\ignore{\textbf{{Image Reconstruction task on USPS Dataset~\citep{usps94}}}:} 

\subsubsection{Image Reconstruction task on USPS Dataset~\citep{usps94}}

The image reconstruction problem~\citep{SOAR09} is given the outer 240 pixel values of a handwritten digit (16x16) from USPS data set, the goal is to predict the 16 pixel values lying in the center. We split the dataset into in 4649 test examples and 4649 training samples (No  knowledge is assumed for the label of the digit). The range of the pixel values in this dataset is in $(-1,1)$.\ignore{

\vspace{5mm}

\noindent{\textbf{Parameter Settings and Error Measures}:}  The parameters $2 \rho_x^2$, $2 \rho_y^2$, $\lambda_X$, and $\lambda_Y$ were assigned to  $2$, $2$, $0.5*10^{-3}$, and $0.5*10^{-3}$, respectively.} The error measure amounts to the root mean-square error averaged over the 16 gray-scales in the center. $Error_{pose}(\hat{y},y^*) =  \|\hat{y} - {y^*}\|$, where $\hat{y} \in R^{16}$ is the predicted 16-values' vector lying in the center, $y^*$ is the true 16-colors of the given outer 240 pixels values $x$. 

\ignore{\textbf{{3D pose estimation task on Poser Dataset~\citep{Trig06}}}:}

\subsubsection{3D pose estimation task on Poser Dataset~\citep{Trig06}}

Poser dataset consists of 1927 training and 418 test images, which are synthetically generated and tuned to unimodal predictions. The image features, corresponding to bag-of-words representation with silhouette-based shape context features. The \textit{TGPs} requires inversion of $N \times N$ matrices during the training, so the complexity of the solution is $O(N^3)$, which is impractical when $N$ is larger. Hence, in both Poser and Human Eva datasets, we applied the \textit{TGPs} by finding the $K_{tr}$ nearest neighbors ($K_{tr} \approx 800$ in our experiments).  This strategy was also adopted in~\citep{Bo:2010,Yamada:2012}.\ignore{

\vspace{5mm}

\noindent{\textbf{Parameter Settings and Error Measures}:} The parameters $2 \rho_x^2$, $2 \rho_y^2$, $\lambda_X$, and $\lambda_Y$ were assigned to  $5$, $5000$, $10^{-4}$, and $10^{-4}$, respectively.} Poser dataset was generated using Poser software package, from motion capture (Mocap)data (54 joint angles per frame). The error is measured by the root mean square error (in degrees), averaged over all joints angles, and is given by: $Error_{pose}(\hat{y}, y^*) = \frac{1}{54} \sum_{m=1}^{54} \| {\hat{y}}^m  - {y^*}^m mod$  $360^\circ \|$ , where $\hat{y} \in R^{54}$ is an estimated pose vector, and $y^* \in R^{54}$ is a true pose vector.

\ignore{\textbf{{3D pose estimation task on HumanEva Dataset~\citep{SigalBB10} }:} }

\subsubsection{3D pose estimation task on HumanEva Dataset~\citep{SigalBB10} }

HumanEva datset contains synchronized multi-view video and Mocap data. It consists of 3 subjects performing multiple activities. We use the histogram of oriented gradient (HoG) features ($\in R^{270}$) proposed in~\citep{Bo:2010}. We use training and validations sub-sets of HumanEva-I and only utilize data from 3 color cameras with a total of 9630 image-pose frames for each camera. This is consistent with experiments in~\citep{Bo:2010,Yamada:2012}. We use half of the data (4815 frames) for training and half (4815 frames) for testing.\ignore{

\vspace{5mm}

\noindent{\textbf{Parameter Settings and Error Measures}:} The parameters $2 \rho_x^2$, $2 \rho_y^2$, $\lambda_X$, and $\lambda_Y$ were assigned to  $5$, $500000$, $10^{-3}$, and $10^{-3}$, respectively.} In HumanEva, pose is encoded by (20) 3D joint markers defined relative to the torso Distal joint in camera-centric coordinate frame, so $y = [y^{(1)}, y^{(2)}, ..., y^{(20)}] \in R^{60}$ and  $y^{(i)} \in R^3$. Error (in $mm$) for each pose is measured as average Euclidean distance: $Error_{pose}(\hat{y},y^*) = \frac{1}{20} \sum_{m=1}^{20} \|\hat{y}^m - {y^*}^m\|$, where $\hat{y}$ is an estimated pose vector, and $y^*$ is a true pose vector.

\subsection{Parameter Settings and Learning $\alpha$ and $\beta$}
\label{sec:lab}
\ignore{As indicated in lemma~\ref{lemma2}, $\beta$ does not theoretically affect the optimization function. However, in practice, we found $\beta$ affects the convergence speed to predict the output, since $SMTGP$ prediction is not in a closed-form expression and we set the maximum number of steps in the quasi-Newton optimizer to $50$ similar to~\cite{Bo:2010}.} 

Each SMTGP prediction is done by optimizing equation 15 by gradient descend with max steps of 50 (like~\cite{Bo:2010}). Since, we proved that $\beta$ is mainly changing the power of the cost function, which theoretically does not affect the prediction, as detailed in Section~\ref{sec:eiganalsis}. Hence, this motivated us to only consider only three values, which are actually edge cases ($\beta=0.99$), ($\beta = 0.5$ for $\beta<1$),  ($\beta = 1.5$ for $\beta>1$). We found that that the role of $\beta$ in practice is mainly affecting the convergence rate and the purpose of cross validation on $\beta$ is to find $\beta$ that converges faster. We found that there is no specific value of $\beta$ that gives the best performance for all the datasets. Hence, we suggest selecting $\beta$ by cross validation like $\alpha$ but for a different purpose.

\ignore{
The observation that the cost function does not depend on beta theoretically motivated us to only consider three values which are actually edge cases ($\beta-0.99$) is Tsallis, ($\beta = 0.5$, $\beta<1$),   ($\beta = 1.5, \beta>1$). We found that the best changing the value of $\beta$  does not change the performance significantly. However, we found that performing cross validation on $\beta$ achieves the best performance. This is since each prediction is done  by optimizing equation 15 by gradient descend  with max steps of $50$ (like~\cite{Bo:2010}). This concludes that the role of $\beta$ is mainly affecting the convergence rate and the purpose of cross validation on $\beta$ is to fine $\beta$ that converges faster.}  

We performed five fold cross validation on $\alpha$ parameters ranging from 0 to 1 step 0.05. While, we selected three values for $\beta$. $\beta \to 1 = 0.99$ in practice\ignore{(equivalent to minimizing R\'enyi divergence)}, $\beta  = 1.5$ (i.e. $\beta   >1$), $\beta  = 0.5$ (i.e. $\beta <1$ \ignore{which is equivalent to minimizing  Tsallis divergence\footnote{since Tsallis is a special case of SM where $\beta = \alpha$ , $\beta<1$}}). Our learning of the parameters covers different divergence measures and select the setting that minimize the error on the validation set. Finally, we initialize $y$ in $SMTGP$ by $KLTGP$ prediction in~\citep{Bo:2010}. Regarding $\lambda_X$, $\lambda_Y$, $\rho_X$ and $\rho_y$, we use the values selected during the training of KLTGP~\citep{Bo:2010}. Table~\ref{tab:tableparamsettings} shows the parameter setting, we used for KLTGP, IKTGP, and SMTGP models.  All these \ignore{TGPs }models share $\rho_x$,  $\rho_y$, $\lambda_x$, and $\lambda_y$ parameters. However, SMTGP has $\alpha$ and $\beta$ as additional parameters. \ignore{ Table~\ref{tab:tableparamsettings} also indicates that our lemma is hypothesis about $\beta$ is true, since $\beta= 0.99$ for all $\beta$, was chosen for all the datasets. This is justified by the results of Lemma~\ref{lemma2}  and that SM-divergence converges to Tsallis divergence as $\beta \to 1$, which is a stable limit for SM. }
\vspace{-5mm}
\begin{table}[htbp]
  \centering
  \caption{Parameter Settings for TGPs}
    \scalebox{0.8}{
    \begin{tabular}{|l|l|l|l|l|l|l|}
    \toprule
          &  $2 \rho_x^2$ & $2 \rho_y^2$ & $\lambda_x$ & $\lambda_y$ & $\alpha$ & $\beta$ \\
    \midrule
    \textbf{Toy 1} & $5$     & $0.05$  & $10^{-4}$ & $10^{-4}$ & $0.9$   & $1.5$ \\
    \textbf{Toy 2} & $5$     & $0.05$  & $10^{-4}$ & $10^{-4}$ & $0.6$   & $0.99$ \\
    \textbf{USPS} & $2$     & $2$     & $0.5*10^{-3}$ & $0.5*10^{-3}$ & $0.9$   & $0.99$ \\
    \textbf{Poser} & $5$     & $5000$  & $10^{-4}$ & $10^{-4}$ & $0.7$   & $0.5$ \\
    \textbf{Heva} & $5$     & $500000$ & $10^{-3}$ & $10^{-3}$ & $0.99$  & $0.99$ \\
    \bottomrule
    \end{tabular}%
    }
  \label{tab:tableparamsettings}%
\end{table}% 
\vspace{-5mm}
\subsection{Results}
\label{sec:exprdisc}
As can be noticed from Figures~\ref{fig:toy1restgp} and~\ref{figtoy1ressmtgp}, SMTGP improved on KLTGP on Toy 1 dataset. Further improvement has been achieved on  Toy 2 dataset, which is more challenging; see Figures~\ref{fig:kltgptoy2} and~\ref{fig:smtgptoy2}. These results indicates the advantages of the parameter selection of $\alpha$ and $\beta$. From Table~\ref{tab:experiments}, we can notice that SMTGP improved on KLTGP by $12.70\%$  and also on  IKLTGP by $3.51\%$ in Toy 2, which shows the adaptation behavior of SMTGP by tuning $\alpha$ and $\beta$. It was argued in~\citep{Bo:2010} that KLTGP performs better than IKLTGP in pose estimation. While, they reported that they gave almost the same performance on Toy 1, which we refer here by Toy 1. We presented Toy 2 to draw two conclusions. First, KLTGP does not always outperform IKLTGP as argued in~\citep{Bo:2010} in HumanEva dataset. Second, SMTGP could be tuned by cross-validation to outperform both KLTGP and IKLTGP.
\begin{figure}[ht!]
  \begin{minipage}[b]{0.45\linewidth}
    \centering
    \includegraphics[width=1.1\textwidth]{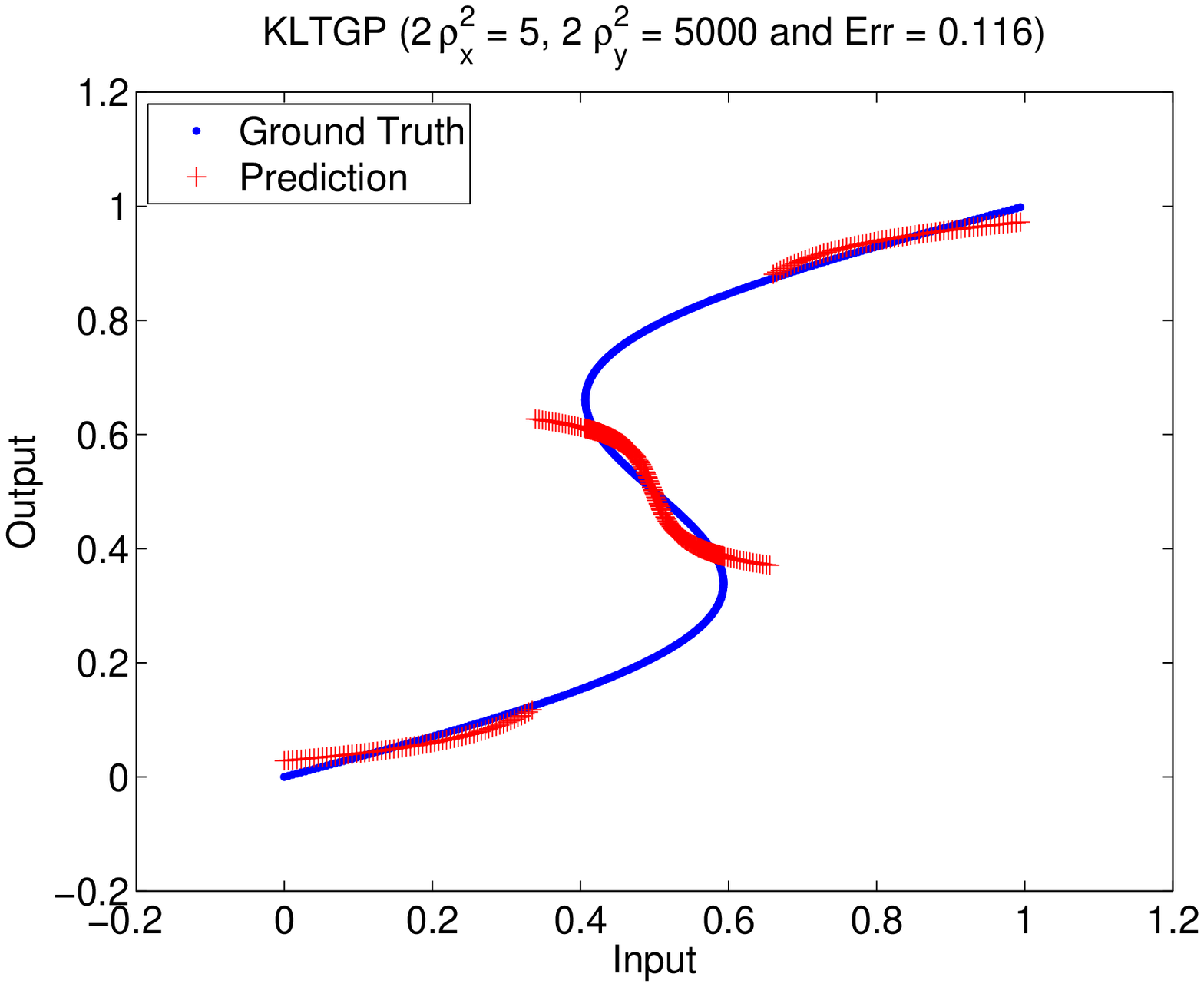}
    \caption{Toy1: KLTGP  \textbf{error  = 0.116 (+/- 0.152)}}
      \label{fig:toy1restgp}
  \end{minipage} 
     \begin{minipage}[b]{0.1\linewidth}   $\textbf{ }$
     \end{minipage}
   \begin{minipage}[b]{0.45\linewidth}
    \centering
     \includegraphics[width=1.1\textwidth]{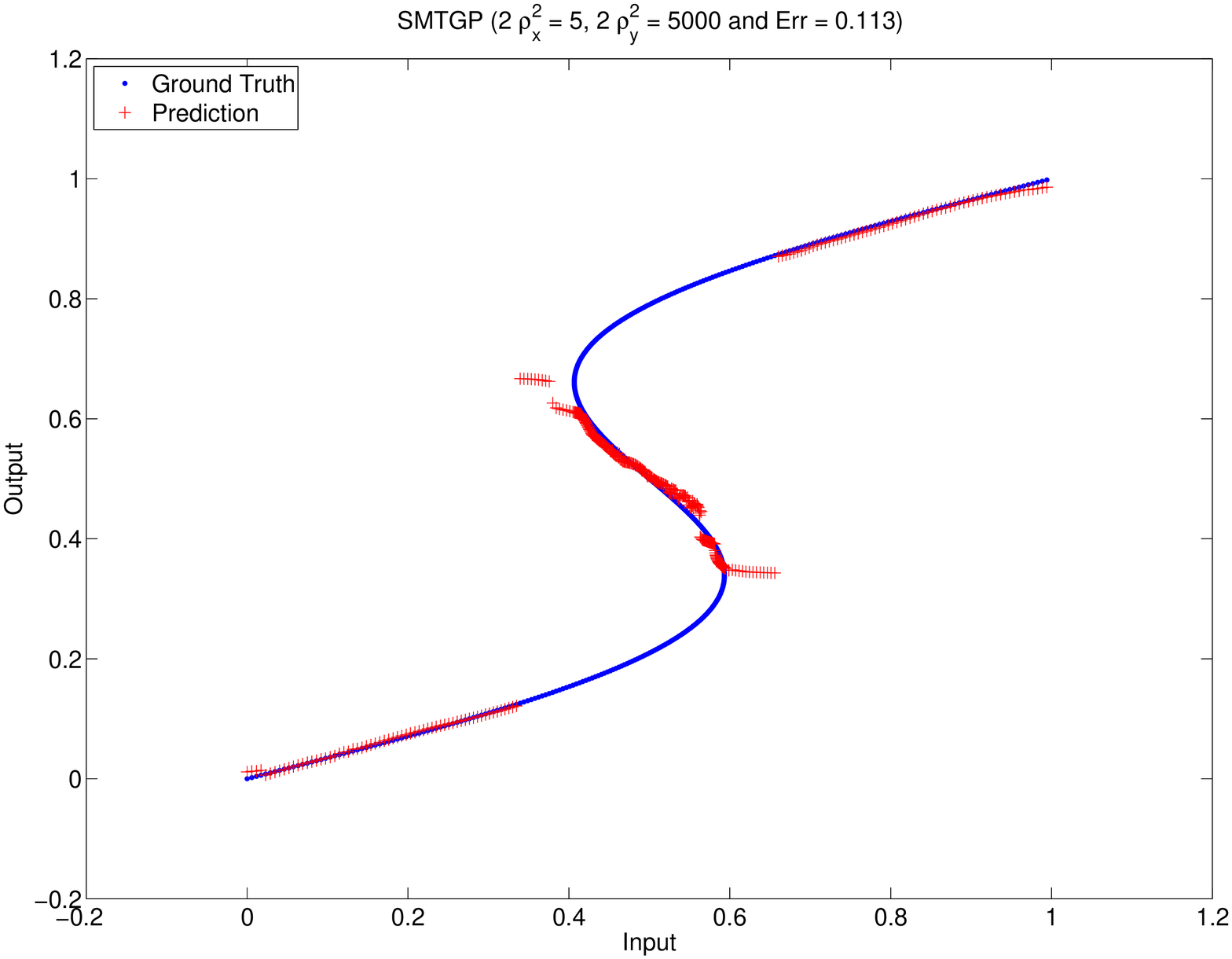}
    \caption{Toy1:SMTGP  \textbf{error  = 0.113  (+/- 0.158)}}
     \label{figtoy1ressmtgp}
  \end{minipage}
  \vspace{-5mm}
\end{figure}
\begin{figure}[ht!]
  \begin{minipage}[b]{0.44\linewidth}
    \centering
    \includegraphics[width=1.1\textwidth]{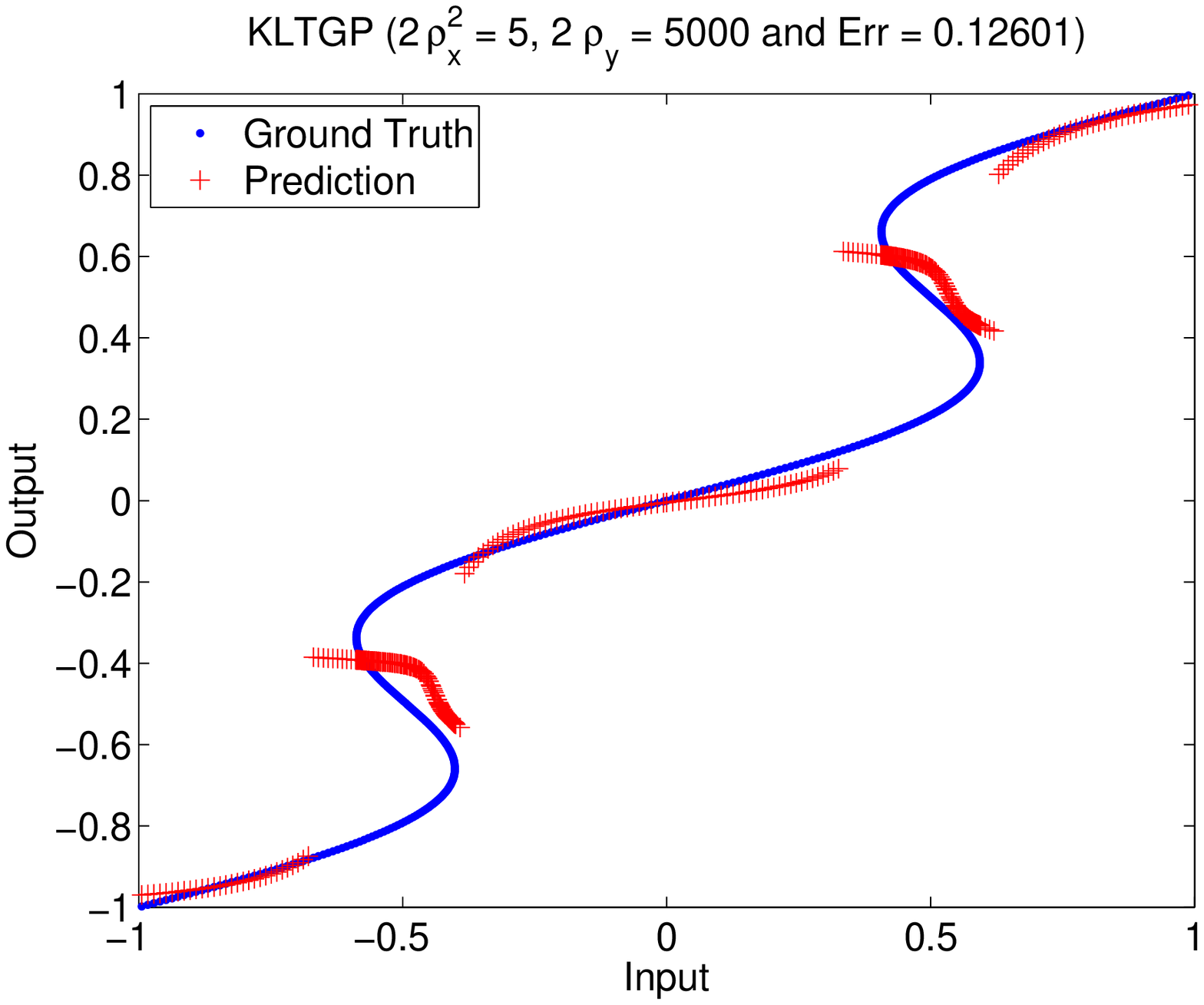}
    \caption{Toy2:KLTGP   \textbf{error  = 0.126 (+/- 0.14)}}
      \label{fig:kltgptoy2}
  \end{minipage} 
    \begin{minipage}[b]{0.1\linewidth}   $\textbf{ }$
     \end{minipage}
   \begin{minipage}[b]{0.44\linewidth}
    \centering
     \includegraphics[width=1.1\textwidth]{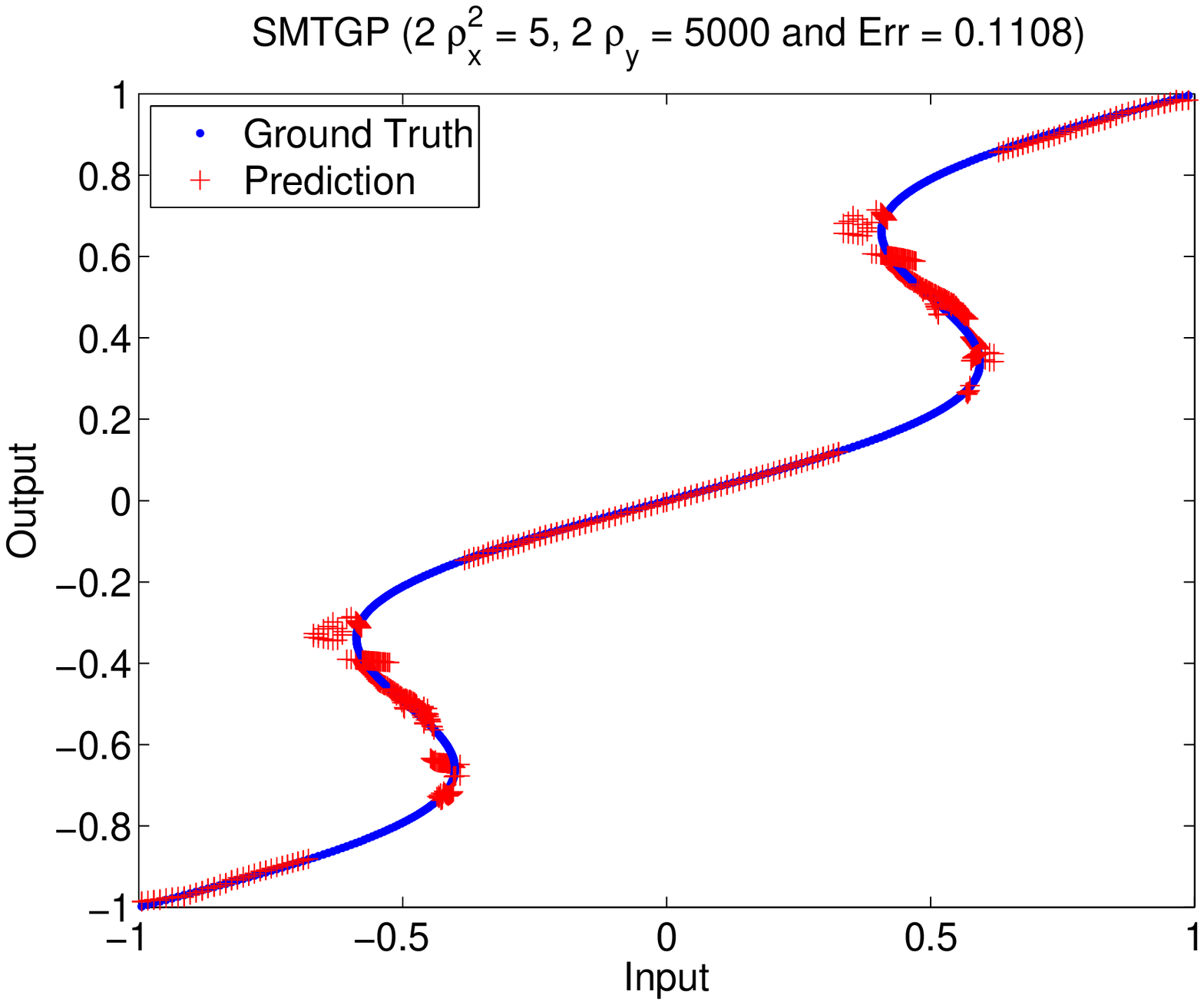}
    \caption{Toy2: SMTGP  \textbf{error  = 0.110 (+/- 0.15 )}}
    \label{fig:smtgptoy2}
  \end{minipage}
  \label{figtoy}
\end{figure}

\vspace{-5mm}
\begin{table}[htbp]
  \centering
  \caption{Regression error and time of SMTGP, KLTGP, and IKLTGP, and error reduction of SMTGP against KLTGP and IKLTGP on the five datasets, imp. denotes the reduction}	
  \scalebox{0.7}{
    \begin{tabular}{|l|l|ll|ll|}
    \toprule
          & \textbf{SMTGP} & \textbf{KLTGP } & \textit{\textbf{Imp.\%}} & \textbf{IKLTGP} & \textit{\textbf{Imp\%}}  \\
          & &~\citep{Bo:2010} & &~\citep{Bo:2010} & \\
    \midrule
    \textbf{Toy1} & \textbf{0.1126 (18.6 sec)} & 0.116 (19.9 sec) & 2.93  \% & {0.115 (25.8 sec)} & {2.09\%}  \\
    \textbf{Toy2} & \textbf{0.11 (20.1 sec) } & 0.126 (19.2 sec) & 12.70 \% & {0.114 (25.1 sec)} & {3.51\%} \\
    \textbf{USPS} & \textbf{0.2587 (1001.7 sec)} & {0.2665 (945 sec)} & {2.93\%} & 0.2683 (1154 sec) & 3.58\% \\
    \textbf{Poser (deg)} & \textbf{5.4296 (104.3 sec)}  & {5.4296 (121.6 sec)} & 0.00\% & 6.484 (146.3 sec) & 16.26\%\\
    \textbf{HEva (mm)} & \textbf{37.59 (1631.6 sec)} & {37.64 (2028.4 sec)} & 0.13\% & 55.622 ( 2344 sec) & 32.42\% \\
    \bottomrule
    \end{tabular}%
    }
  \label{tab:experiments}%
\end{table}%

Another important observation in Table~\ref{tab:experiments} is that KLTGP outperforms IKLTGP on Poser and HumanEva datasets, while IKLTGP outperform KLTGP in the toy examples (slightly in the first and significantly in the second). The interesting behavior is that SMTGP performs at least as good as the best of KLTGP and IKLTGP in all of the datasets. KLTGP and IKLTGP are biased towards one of the input and the output distributions. However, SMTGP learns from the training data the bias factor (using $\alpha$) towards the input or the output distributions. These results could also be justified by the fact that SM divergence is a generalization of a family of divergence measure. A powerful property in SMTGP is that by controlling $\alpha$ and $\beta$, SMTGP provides a set of divergence functions to optimize for prediction. However, a member of this set is selected during training by tuning $\alpha$ and $\beta$ on a validation set. Hence, SMTGP learns $\alpha$ and $\beta$ to make better predictions. Finally, SMTGP has a desirable generalization  on the test set; see Table~\ref{tab:experiments}. Table~\ref{tab:experiments} also shows that SMTGP does not only have same complexity as KLTGP but also it has a similar constant factor. In four of the datasets, SMTGP is faster than IKLTGP and KLTGP \footnote{ for KLTGP, we used the implementation provided by~\cite{Bo:2010}}. We optimized the matrix operations in three methods as possible. SMTGP and KLTGP have similar number of matrix operations; this justifies why they have similar computational times.

We conclude our results by reporting the performance of GPR, HSIC-KNN, KTA-KNN, and W-KNN on the five datasets\footnote{ These baseline approaches was also compared in~\citep{Bo:2010} against KLTGP, and our results  is consistent with the conclusion that we reached from the comparison but only on Toy Example 1 and HumanEva dataset; see~\citep{Bo:2010} for more about the parameters of these baselines and its selection. KNN indicates that these methods were applied to training data in K-neighborhood of the testing point}; see Table~\ref{tab:exprbaselines}. Comparing   Table~\ref{tab:experiments} to  Table~\ref{tab:exprbaselines}, it is obvious that TGPs outperforms GPR, HSIC-KNN, KTA-KNN, and W-KNN.

% Table generated by Excel2LaTeX from sheet 'Sheet1
\vspace{-5mm}'
\begin{table}[htbp]
  \centering
  \caption{Regression error of GPR, HSIC-KNN, KTA-KNN, and W-KNN regression models}
  \scalebox{0.62}{
    \begin{tabular}{|l|l|l|l|l|}
    \toprule
          & \textbf{GPR} & \textbf{WKNN } & \textbf{HSICKNN  } & \textbf{KTAKNN  } \\
          &~\citep{Rasmussen:2005} &~\citep{Rasmussen:2005} &~\citep{HSIC05} &~\citep{KTA01} \\
    \midrule
    \textbf{Toy1} & 0.17603 & 0.15152 & 0.18396 & 0.19333 \\
    \textbf{Toy2} & 0.19011 & 0.16986 & 0.21294 & 0.19134 \\
    \textbf{USPS} & 0.31504 & 0.2731 & 0.26832 & 0.26679 \\
    \textbf{Poser (deg)} & 6.0763 & 5.62  & 7.1667 & 8.4739 \\
    \textbf{HEva (mm)} & 46.6987 & 53.0834 & 57.8221 & 57.8733 \\
    \bottomrule
    \end{tabular}%
    }
  \label{tab:exprbaselines}%
\end{table}%

%% file: conclusion.tex
We proposed a framework for structured output regression based on SM-divergence.  We performed a  theoretical analysis to understand the properties of SMTGP prediction, which helped us learn $\alpha$ and $\beta$ parameters of SM-divergence. As a part of our analysis, we argued on a certainty measure that could be associated with each prediction. We here discuss these main findings of our work.

A critical theoretical aspect that is missing in the KL-based TGP formulation is understanding the cost function from regression-perspective. We cover this missing theory not only by analyzing the cost function based on KL, but instead, by providing an understanding of SMTGP cost function, which covers (KL, Renye, Tsallis, Bhattacharyya as special cases of its parameters). Our claims are supported by a theoretical analysis, presented in Section~\ref{sec:eiganalsis}. The main theoretical result is that SM-based TGP (SMTGP) prediction maximizes a certainty measure, we call $\varphi_{\alpha}(x,y) $, and the prediction does not depend on $\beta$ theoretically. A  probabilistic interpretation of  $\varphi_{\alpha}(x,y) $ was discussed as part of our analysis and it was shown to have a  negative correlation with the test error,  which is an interesting result; see figure ~\ref{fig:phixy}. The figure highlights the similarity between this SMTGP certainty measure  and predictive variance provided by Gaussian Process Regression (GPR) ~\citep{Rasmussen:2005} for single output prediction. A computationally efficient closed-form expression for SM-divergence was presented, which leads to \ignore{ two main results; see section  ~\ref{sec:smTGP}. First, SMTGP prediction complexity is} reducing  SMTGP prediction complexity from $O(N^3)$ to $O(N^2)$\footnote{$N$ is the number of the training points}; this makes SMTGP and KLTGP computationally equivalent.  Moreover, it reduces the number of operations to compute SM-divergence between two general Gaussian distributions, out of TGP context; see section  ~\ref{sec:smTGP}. \ignore{ Second, the simplified closed form expression reduces  the number of operations to compute SM-divergence between two  Gaussian Distributions from $\frac{5 N^3}{3}$ to $N^3$ operations if $\Delta \mu =0$, and from $2 N^3$ to $\frac{4 N^3}{3}$ operations, if $\Delta \mu \neq 0$.}\ignore{ Hence, it is important to mention that sections  ~\ref{sec:eiganalsis} and ~\ref{sec:smTGP} are major parts of our contribution.} Practically, \ignore{ we proposed a  structured regression framework based on SM Divergence. } we achieve  structured output regression by \ignore{.Our theoretical anlysis   helped us learn}tuning $\alpha$ and $\beta$ parameters of SM-divergence through cross validation under SMTGP cost function. We performed  an intensive evaluation of different tasks on five datasets and we experimentally observed a desirable generalization property of SMTGP. Our experiments report that our resultant approach, SMTGP, outperformed KLTGP, IKLTGP, GPR, HSIC, KTA, and W-KNN methods on two toy examples and three datasets.

We conclude by highlighting a  practical limitation of SMTGP, which is that it requires an additional time for tuning $\alpha$ and $\beta$ by cross validation. However, we would like to indicate that this cross  validation time is very short for the datasets ($0.9$ hour for poser dataset and 14 hours for Human Eva dataset). Using a smaller grid could significantly decrease this validation time. We used a grid of $20$ steps for $\alpha$. However, we found that in our experiments it is enough to use grid of size $10$ (step $0.1$ instead of $0.05$). In addition, selecting a single randomly selected validation set like Neural networks models could save a lot of time instead of selecting $\alpha$ and $\beta$ on the entire training set by cross validation, which we performed in our experiment.

\section{Conclusion}

We presented a  theoretical analysis of a two-parameter generalized divergence measure, named Sharma-Mittal(SM), for structured output prediction. We proposed an alternative, yet equivalent, formulation for SM divergence whose computation is quadratic compared to cubic for the structured output prediction task (Lemma~\ref{lemma11}). We further investigated theoretical properties which is concluded by a probabilistic
causality direction of our SM objective function; see Section~\ref{sec:eiganalsis}. We performed extensive experiments to validate our findings on different tasks and datasets (two datasets for pose estimation, one dataset for image reconstruction and two toy examples). 

\textbf{Acknowledgment. } This research was partially funded by NSF award \# 1409683.

%% file: appendices.tex
\subsection*{Appendix A: Relationship between $K_{X\cup x}^{-1}$ and $K_{X}^{-1}$}

$K_{X\cup x}^{-1}$ is $O(N^2)$ to compute, given that the signular value decomposition of $K_X$ is precomputed during the training, from which $K_X^{-1}$ and $K_X^{-2}$ are computed as well.  Then,  applying the matrix inversion lemma  ~\citep{minvlemma99}, $K_{X\cup x}^{-1}$ could be related to $K_{X}^{-1}$ as follows
\ignore{
\begin{equation}
K_{X\cup x}^{-1} = \begin{bmatrix}
K_X^{-1} +  \frac{{K_X^x}^T K_X^{-2}K_X^x}{k_X(x,x) - {K_X^x}^T K_X^{-1}K_X^x}& \frac{-K_X^{-1} K_X^x}{k_X(x,x) - {K_X^x}^T K_X^{-1} K_X^x} \\
\frac{-{K_X^x}^T K_X^{-1}}{k_X(x,x) - {K_X^x}^T K_X^{-1} K_X^x}  & \frac{1}{k_X(x,x) - {K_X^x}^T K_X^{-1}K_X^x} \\
\end{bmatrix}
\label{eqekxkxux}
\end{equation}}
\begin{equation}
{K_{X \cup x}}^{-1}  =   \begin{bmatrix}
K_X^{-1} +\frac{1}{c_x} K_X^{-1} K_X^x {K_X^x}^T K_X^{-1} & \frac{-1}{c_x}  K_X^{-1}  K_X^x\\ 
 \frac{-1}{c_x}  {K_X^x}^T K_X^{-1} & \frac{1}{c_x}
\end{bmatrix}
\label{eqekxkxux}
\end{equation}
where $c_x= \it{K_X}(x,x) -  {K_X^x}^T K_X^{-1}  K_X^x$. Given that $K_X^{-1}$ and $K_X^{-2}$ are already computed,  then computing $K_{X\cup x}^{-1}$ becomes $O(N^2)$ using equation ~\ref{eqekxkxux}. This equation applied to any kernel matrix (\ie     relating $K_{Y\cup y}^{-1}$ to $K_{Y}^{-1}$) . 

\ignore{
Since   $K_X^{-1}$ is already computed during the training, $K_{X\cup x}^{-1}$ is computed in $O(N^2)$ at test time. This equation applied to any kernel matrix (\ie $K_{Y\cup y}^{-1}$ is related to $K_{Y}^{-1}$ \etc).}

\subsection*{Appendix B: SM $D$ TGP}

\subsubsection*{{Cost Function}}
\begin{equation}
\small
\begin{split}
& D_{\alpha, \beta}(p(X,x) : p(Y,y)) = 
\frac{1}{\beta-1} \Bigg[\Big(\frac{|K_{X\cup x}|^{\alpha} {|K_{Y \cup y}|}^{1 -\alpha}}{|(\alpha K_{X\cup x}^{-1} +(1-\alpha) {K_{Y \cup y}}^{-1})^{-1}|}\Big)^{-\frac{1-\beta}{2(1-\alpha)}}  -1\Bigg] 
\end{split}
\end{equation}

\begin{equation}
\small
\begin{split}
L_{\alpha, \beta}(p(X,x) : p(Y,y)) = & \frac{1}{\beta-1} (k_Y(y,y) - {K_Y^y}^T {K_Y}^{-1} {K_Y^y})^\frac{-(1-\beta)}{2}     \cdot     \\
& |(\alpha K_{X\cup x}^{-1} +(1-\alpha) {K_{Y \cup y}}^{-1})|^\frac{-(1-\beta)}{2(1-\alpha)} \\
\end{split}
\end{equation}

From the matrix inversion Lemma, 

\begin{equation*}
\small
\begin{split}
&{K_{X \cup x}}^{-1}  =   \begin{bmatrix}
K_X^{-1} +\frac{1}{c_x} K_X^{-1} K_X^x {K_X^x}^T K_X^{-1} & \frac{-1}{c_x}  K_X^{-1}  K_X^x\\ 
 \frac{-1}{c_x}  {K_X^x}^T K_X^{-1} & \frac{1}{c_x}
\end{bmatrix}\\
&{K_{Y \cup y}}^{-1}  =   \begin{bmatrix}
K_Y^{-1} +\frac{1}{c_y} K_Y^{-1} K_Y^y {K_Y^y}^T K_Y^{-1} & \frac{-1}{c_y}  K_Y^{-1}  K_Y^y\\ 
 \frac{-1}{c_y}  {K_Y^y}^T K_Y^{-1} & \frac{1}{c_y}
\end{bmatrix}\\
\end{split}
\end{equation*}
where $c_x= \it{K_X}(x,x) -  {K_X^x}^T K_X^{-1}  K_X^x$, $c_y= \it{K_Y}(y,y) -  {K_Y^y}^T K_Y^{-1}  K_Y^y$

function evaluation  could be computed in $O (N^2)$ where $N$ is the number of points in the training set.

\begin{equation}
\small
\begin{split}
 log L_{\alpha, \beta}(p(X,x) : p(Y,y)) = &
 \frac{-(1-\beta)}{2} \cdot log (k_Y(y,y) - {K_Y^y}^T {K_Y}^{-1} {K_Y^y})     +    \\
& \frac{-(1-\beta)}{2(1-\alpha)} \cdot log |(\alpha K_{X\cup x}^{-1} +(1-\alpha) {K_{Y \cup y}}^{-1})| \\
\end{split}
\end{equation}

\subsubsection*{Gradient Calculation}

\begin{comment}
\begin{equation}
\small
\begin{split}
&\frac{\partial log L(\alpha,\beta)}{\partial y(d)}  = \frac{-(1-\beta)}{2}  \frac{(\frac{\partial k_Y(y,y)}{\partial y(d)} - 2\cdot {K_Y^y}^T {K_Y}^{-1} \frac{\partial k_Y(y,y)}{\partial y(d)})}{(k_Y(y,y) - {K_Y^y}^T {K_Y}^{-1} {K_Y^y})}  \\
&+  \frac{-(1-\beta)}{2(1-\alpha)} \cdot \frac{ |(\alpha K_{X\cup x}^{-1} +(1-\alpha) {K_{Y \cup y}}^{-1})| }{ |(\alpha K_{X\cup x}^{-1} +(1-\alpha) {K_{Y \cup y}}^{-1})| } \cdot  \\
& -Tr\Big( \big(\alpha K_{X\cup x}^{-1} +(1-\alpha) {K_{Y \cup y}}^{-1} \big) \cdot \\
&  \big(\alpha K_{X\cup x}^{-1} +(1-\alpha) {K_{Y \cup y}}^{-1} \big)^{-1} \\
&   (1-\alpha)\cdot \frac{\partial K_{Y \cup y}}{\partial y(d)}    \\
&\big(\alpha K_{X\cup x}^{-1} +(1-\alpha) {K_{Y \cup y}}^{-1} \big)^{-1}    \Big)\\
\end{split}
\end{equation}
\end{comment}

Following matrix calculus, $\frac{\partial log L(\alpha,\beta)}{\partial y(d)}$ could be expressed as follows 

\begin{equation}
\small
\begin{split}
\frac{\partial log L(\alpha,\beta)}{\partial y(d)}  = & \frac{-(1-\beta)}{2}  \frac{(\frac{\partial k_Y(y,y)}{\partial y(d)} - 2 {K_Y^y}^T {K_Y}^{-1} \frac{\partial K_Y^y}{\partial y(d)})}{(k_Y(y,y) - {K_Y^y}^T {K_Y}^{-1} {K_Y^y})}  + \\
&  \frac{-(1-\beta)}{2(1-\alpha)}  (1-\alpha) Tr \Big( \frac{\partial {K_{Y \cup y}}^{-1}}{\partial y(d)}   \cdot 
 \big(\alpha K_{X\cup x}^{-1} +(1-\alpha) {K_{Y \cup y}}^{-1} \big)^{-1}    \Big) \\
\end{split}
\end{equation}

Since  $\frac{\partial k_Y(y,y)}{\partial y(d)} = 0$ for rbf-kernels  and $\frac{-(1-\beta)}{2(1-\alpha)}  (1-\alpha) =  \frac{-(1-\beta)}{2}$, then

\begin{equation}
\small
\begin{split}
\frac{\partial log L(\alpha,\beta)}{\partial y(d)}  = & \frac{-(1-\beta)}{2}  \frac{( - 2 {K_Y^y}^T {K_Y}^{-1} \frac{\partial K_Y^y}{\partial y(d)})}{(k_Y(y,y) - {K_Y^y}^T {K_Y}^{-1} {K_Y^y})}  -  \\
&  \frac{-(1-\beta)}{2}  Tr \Big(  \big(\alpha K_{X\cup x}^{-1} +(1-\alpha) {K_{Y \cup y}}^{-1} \big)^{-1} \cdot {K_{Y \cup y}}^{-1} \frac{\partial K_{Y \cup y}}{\partial y(d)} {K_{Y \cup y}}^{-1}  \Big) 
\end{split}
\end{equation}

By factorization, the gradient could be further simplified into the following form.

\begin{equation}
\small
\begin{split}
\frac{\partial log L(\alpha,\beta)}{\partial y(d)}  = & \frac{-(1-\beta)}{2}  \frac{( - 2 {K_Y^y}^T {K_Y}^{-1} \frac{\partial K_Y^y}{\partial y(d)})}{(k_Y(y,y) - {K_Y^y}^T {K_Y}^{-1} {K_Y^y})}  \\
&-  \frac{-(1-\beta)}{2}  Tr \Big( {K_{Y \cup y}}^{-1} \big(\alpha K_{X\cup x}^{-1} +(1-\alpha) {K_{Y \cup y}}^{-1} \big)^{-1} \cdot  {K_{Y \cup y}}^{-1} \frac{\partial K_{Y \cup y}}{\partial y(d)}   \Big) 
\end{split}
\end{equation}

Since $(A B)^{-1} = B^{-1} A^{-1}$, where $A$ and $B$ are invertible matrices, then 

\begin{equation}
\small
\begin{split}
\frac{\partial log L(\alpha,\beta)}{\partial y(d)}  = & \frac{-(1-\beta)}{2}  \frac{( - 2 {K_Y^y}^T {K_Y}^{-1} \frac{\partial K_Y^y}{\partial y(d)})}{(k_Y(y,y) - {K_Y^y}^T {K_Y}^{-1} {K_Y^y})}  \\
&-  \frac{-(1-\beta)}{2}  Tr \Big(  \big(  K_{Y \cup y} \big(\alpha K_{X\cup x}^{-1} +(1-\alpha) {K_{Y \cup y}}^{-1} \big) K_{Y \cup y}\big)^{-1} \cdot  \frac{\partial K_{Y \cup y}}{\partial y(d)}   \Big) 
\end{split}
\end{equation}

Having applied matrix multiplications, then

\begin{equation}
\small
\begin{split}
\frac{\partial log L(\alpha,\beta)}{\partial y(d)}  = & \frac{-(1-\beta)}{2}  \frac{( - 2 {K_Y^y}^T {K_Y}^{-1} \frac{\partial K_Y^y}{\partial y(d)})}{(k_Y(y,y) - {K_Y^y}^T {K_Y}^{-1} {K_Y^y})}  \\
&-  \frac{-(1-\beta)}{2}  Tr \Big(  \big(\alpha K_{Y \cup y} K_{X\cup x}^{-1} K_{Y \cup y} +(1-\alpha) {K_{Y \cup y}} \big)^{-1} \cdot   \frac{\partial K_{Y \cup y}}{\partial y(d)}   \Big) \\
\end{split}
\label{eq111}
\end{equation}

where \begin{math}
\small
\frac{\partial K_{Y \cup y}}{\partial y(d)}  = \begin{bmatrix}
0 & \frac{\partial K_Y^y}{\partial y(d)}\\ 
\frac{\partial {K_Y^y}^T}{\partial y(d)} & 0 \\
\end{bmatrix}\\
\end{math}

Having analyzed equation ~\ref{eq111}, it is not hard to see that 

\begin{equation}
\small
Tr \Big(  \big(\alpha K_{Y \cup y} K_{X\cup x}^{-1} K_{Y \cup y} +(1-\alpha) {K_{Y \cup y}} \big)^{-1} \cdot   \frac{\partial K_{Y \cup y}}{\partial y(d)}   \Big) = 2  \cdot \mu_y^T  \cdot  \frac{\partial {K_Y^y}}{\partial y(d)}
\end{equation}  

where
\begin{math}
\small
 \Big(\alpha K_{Y \cup y} K_{X\cup x}^{-1} K_{Y \cup y} +(1-\alpha) {K_{Y \cup y}} \Big)\mu^{'}_y = [0,0,...0,1]^T
\end{math},  $\mu_y$ is a vector of all elements in $\mu^{'}_y$ except the last element. Hence,

\begin{equation}
\small
\begin{split}
\frac{\partial log L(\alpha,\beta)}{\partial y(d)}  = & \frac{-(1-\beta)}{2}  \frac{(- 2 \cdot {K_Y^y}^T {K_Y}^{-1} \frac{\partial K_Y^y}{\partial y(d)})}{(k_Y(y,y) - {K_Y^y}^T {K_Y}^{-1} {K_Y^y})}  -  \frac{-(1-\beta)}{2} \cdot 2  \cdot \mu_y^T  \cdot  \frac{\partial {K_Y^y}}{\partial y(d)}  \\
\end{split}
\end{equation}

Which directly leads to the final form of $\frac{\partial \log L(\alpha,\beta)}{\partial y^{(d)}}$ 

\begin{equation}
\small
\begin{split}
%&\frac{\partial \log L(\alpha,\beta)}{\partial y^{(d)}}  = \frac{-(1-\beta)}{2}  \frac{(- 2 \cdot {K_Y^y}^T {K_Y}^{-1} \frac{\partial K_Y^y}{\partial y^{(d)}})}{(k_Y(y,y) - {K_Y^y}^T {K_Y}^{-1} {K_Y^y})}  -  \frac{-(1-\beta)}{2} \cdot 2  \cdot \mu_y^T  \cdot  \frac{\partial {K_Y^y}}{\partial y^{(d)}}  \\
&\frac{\partial \log L(\alpha,\beta)}{\partial y^{(d)}}  = {(1-\beta)} \bigg[ \frac{ {K_Y^y}^T {K_Y}^{-1} \frac{\partial K_Y^y}{\partial y^{(d)}}}{(k_Y(y,y) - {K_Y^y}^T {K_Y}^{-1} {K_Y^y})}  +   \mu_y^T  \cdot  \frac{\partial {K_Y^y}}{\partial y^{(d)}} \bigg]  \\
\end{split}
\end{equation}

\subsection*{Appendix C: SM $D'$ TGP}
This derivation is much more simpler starting from our simplified closed-form expression of SM-divergence between two multivariate Gaussians. After ignoring multiplied positive constants and added constants (i.e. $|K_X|$, $|K_Y|$ ,$|(1-\alpha) K_X + \alpha K_Y|$ are multiplied constants, $-\frac{1}{1-\beta}$) is an added constant), the improved SMTGP cost function reduces to
 \begin{equation}
\small
\begin{split}
 L'_{\alpha, \beta}(p(X,x) : p(Y,y)) =  &\frac{1}{\beta-1} \Big[ {(k_Y(y,y) - {K_Y^y}^T {K_Y}^{-1} {K_Y^y})^\frac{-(1 -\beta)}{2}} \cdot \\
&({{{K_{xy}}^{\alpha} - {K_{XY}^{xy}}^T ((1-\alpha) K_X + \alpha K_Y)^{-1} {K_{XY}^{xy}})}^{\frac{1-\beta}{2(1-\alpha)}} } \Big]
\end{split}
\label{eq:sm2tgp}
\end{equation}
where ${K_{xy}}^{\alpha} = (1-\alpha) k_X(x,x) + \alpha k_Y(y,y)$, $K_{XY}^{xy} = (1-\alpha) K_X^x+  \alpha K_Y^y$. Since the cost function have two factors that does depend on $y$, we follow the rule that if $g(y) = f(y) r(y)$, then  $g'(y) = f'(y) r(y)+f(y)r'(y)$, which interprets the two terms the derived gradient below.
\begin{equation}
\small
\begin{split}
\frac{\partial L'(\alpha,\beta)}{\partial y^{(d)}}  (p(X,x) : p(Y,y))= & \frac{1}{\beta-1} \Big( {\frac{-(1 -\beta)}{2} (k_Y(y,y) - {K_Y^y}^T {K_Y}^{-1} {K_Y^y})^\frac{-(2 -\beta)}{2}} \cdot \\& {(\frac{\partial k_Y(y,y)}{\partial y^{(d)}} - 2 \cdot {K_Y^y}^T {K_Y}^{-1} \frac{\partial K_Y^y}{\partial y^{(d)}})}  \cdot\\
&({{{K_{xy}}^{\alpha} - {K_{XY}^{xy}}^T ((1-\alpha) K_X + \alpha K_Y)^{-1} {K_{XY}^{xy}})}^{\frac{1-\beta}{2(1-\alpha)}} } +\\
&{(k_Y(y,y) - {K_Y^y}^T {K_Y}^{-1} {K_Y^y})^\frac{-(1 -\beta)}{2}} \cdot \frac{1-\beta}{2(1-\alpha)}\\ & ( {{{K_{xy}}^{\alpha} - {K_{XY}^{xy}}^T ((1-\alpha) K_X + \alpha K_Y)^{-1} {K_{XY}^{xy}})}^{\frac{-\beta}{2(1-\alpha)}} }\cdot \\
&(\frac{\partial k_Y(y,y)}{\partial y^{(d)}} - 2 \cdot {K_{XY}^{xy}}^T ((1-\alpha) K_X + \alpha K_Y)^{-1} \cdot \alpha \frac{\partial K_Y^y}{\partial y^{(d)}}) \Big)
\end{split}
\label{eqgradsmd1}
\end{equation}

\subsection*{Appendix D: Advantage of computing SM divergence between two Multivariate Gaussians using Lemma ~\ref{lemma11}}

%\mathbf{\Delta \mu \eq 0:}

As far as we know, an efficient way to compute  $D_{\alpha,\beta}(\mathcal{N}_p, \mathcal{N}_q)$ in  equation~\ref{eqcfsm1} where  $\Delta \mu =0$,\footnote{derived directly from the closed form in~\citep{SM:2012}},  requires $\approx \frac{5 N^3}{3} $ operations; we illustrate as follows. Cholesky decompistion of $\Sigma_p$ and $\Sigma_q$ requires $ \frac{2 N^3}{3}$ operations, with additional $\frac{2 N^3}{3}$ operations for computing $\Sigma_p^{-1}$  and   $\Sigma_q^{-1}$ from the computed decompositions ~\citep{chol97}. Then, choseskly decompition of $\alpha  \Sigma_p^{-1}+ (1-\alpha) \Sigma_q^{-1}$ is computed in additional $\frac{N^3}{3}$ operations. From the computed decompositions, $|\Sigma_p|$, $|\Sigma_q|$ and $|\alpha  \Sigma_p^{-1}+ (1-\alpha) \Sigma_q^{-1}|^{-1} = |(\alpha  \Sigma_p^{-1}+ (1-\alpha) \Sigma_q^{-1})^{-1}| $ are computed in $3 N$ operations, which we ignore. Hence, the required computations for 
 $D_{\alpha,\beta}(\mathcal{N}_p, \mathcal{N}_q)$ are $\frac{2 N^3}{3}+ \frac{2 N^3}{3} + \frac{N^3}{3}  = \frac{5 N^3}{3}$ operations if $\Delta \mu =0$. In case $\Delta \mu \neq 0$, an additional $\frac{N^3}{3}$ operations are required to compute $(\alpha  \Sigma_p^{-1}+ (1-\alpha) \Sigma_q^{-1})^{-1}$, which leads to total of $\frac{6 N^3}{3} = 2 N^3$ operations\footnote{There are additional $N^2$ (matrix vector multiplication), and $N$ (dot product operations), which we ignore since they are not cubic}.
 
In contrast to $D_{\alpha,\beta}(\mathcal{N}_p, \mathcal{N}_q)$, $D'_{\alpha,\beta}(\mathcal{N}_p, \mathcal{N}_q)$ in lemma~\ref{lemma11} could be computed similarly in only $N^3$ operations, if $\Delta \mu =0$, required to compute the determinants of $\Sigma_p$, $\Sigma_q$, and $\alpha \Sigma_q +(1-\alpha) \Sigma_p$ by Cholesky decomposition. In case  $\Delta \mu \neq 0$, an additional $\frac{N^3}{3}$ operations are needed to compute $(\alpha {\Sigma_q} +(1-\alpha) {\Sigma_q})^{-1}$ \footnote{Additional $2 \cdot {O(N^{2.33})}$ for 2 matrix multiplications are ignored}. So, total of $\frac{4 N^3}{3}$ operations are needed if $\Delta \mu \neq 0$. \ignore{This leads to that  $D'_{\alpha,\beta}(\mathcal{N}_p, \mathcal{N}_q)$ 
is $1.67$ times faster to compute than $D_{\alpha,\beta}(\mathcal{N}_p, \mathcal{N}_q)$ under $\Delta \mu =0$ condition, $1.5$ times faster otherwise; see Appendix D for proof.}\ignore{
 3 determinant computations  and 2 matrix inversions, if $\Delta \mu = 0$. However, equation~\ref{eq:ncfe}. requires only 3 matrix determinant computations under the same condition. This simplification could be used to efficiently compute SM divergence between two Gaussian Distributions, out of the context of TGPs. 
}Accordingly,  $D'_{\alpha,\beta}(\mathcal{N}_p, \mathcal{N}_q)$ is 1.67 ( $\frac{5 N^2}{3}$   / $N^3$) times faster to compute than $D_{\alpha,\beta}(\mathcal{N}_p, \mathcal{N}_q)$ if  $\Delta \mu =0$, and 1.5 ( ${2 N^3}$   / $\frac{4 N^3}{3}$) times faster, otherwise.